\documentclass[nohyperref]{article}

\usepackage{microtype}
\usepackage{graphicx}
\usepackage[export]{adjustbox}
\usepackage{amsmath}
\usepackage{amsthm}
\usepackage{graphicx}
\usepackage{subcaption}
\usepackage{graphicx}
\usepackage{caption}

\usepackage[ruled,linesnumbered,noresetcount,vlined]{algorithm2e}

\newtheorem{problem*}{Problem}

\newtheorem{theorem}{Theorem}

\newtheorem{definition}{Definition}
\newtheorem{proposition}{Proposition}

 \newcommand{\cY}{\mathcal{Y}}
 \newcommand{\cX}{\mathcal{X}}

\DeclareMathOperator*{\argmax}{argmax}

\usepackage{booktabs} 
\usepackage{bm}
\usepackage{xcolor}
\usepackage[most]{tcolorbox}
\usepackage{nicefrac}

\usepackage{hyperref}

\usepackage[accepted]{icml2022}

\usepackage{amsmath}
\usepackage{amssymb}
\usepackage{mathtools}
\usepackage{amsthm}

\usepackage[capitalize,noabbrev]{cleveref}

\theoremstyle{plain}

\icmltitlerunning{Personalized Privacy Auditing and Optimization at Test Time \hfill\thepage}

\allowdisplaybreaks
\begin{document}

\twocolumn[
\icmltitle{Personalized Privacy Auditing and Optimization at Test Time}



\icmlsetsymbol{equal}{*}

\begin{icmlauthorlist}
\icmlauthor{Cuong Tran}{yyy}
\icmlauthor{Ferdinando Fioretto}{yyy}
\end{icmlauthorlist}
\icmlaffiliation{yyy}{Department of Computer Science, Syracuse University, Syracuse NY, USA}
\icmlcorrespondingauthor{Cuong Tran}{ctran@syr.edu}
\icmlcorrespondingauthor{Ferdinando Fioretto}{ffiorett@syr.edu}

\icmlkeywords{Privacy, Machine Learning}

\vskip 0.3in
]



\printAffiliationsAndNotice{}  

\begin{abstract}
A number of learning models used in consequential domains, such as to assist in legal, banking, hiring, and healthcare decisions, make use of potentially sensitive users' information to carry out inference. Further, the complete set of features is typically required to perform inference. This not only poses severe privacy risks for the individuals using the learning systems, but also requires companies and organizations massive human efforts to verify the correctness of the released information. 

This paper asks whether it is necessary to require \emph{all} input features for a model to return accurate predictions at test time and shows that, under a personalized setting, each individual may need to release only a small subset of these features without impacting the final decisions.
The paper also provides an efficient sequential algorithm that chooses which attributes should be provided by each individual. Evaluation over several learning tasks shows that individuals may be able to report as little as 10\% of their information to ensure the same level of accuracy of a model that uses the complete users' information.
\end{abstract}

\section{Introduction}
\label{sec:intro}

The remarkable success of learning models also brought with it pressing challenges at the interface of privacy and decision-making. Privacy, in particular, has been cited as one of the most pressing challenges of modern machine learning systems \cite{papernot2016towards}. 
The requirement to protect personally identifiable information is especially important as machine learning pipelines become routinely adopted to guide consequential decisions, such as to assist in legal processes, banking, hiring, and healthcare decisions. 


To contrast this challenge, several privacy-enhancing technologies have been proposed in the last decades. 
Among these \emph{Differential Privacy} \cite{Dwork:06} has found its place as a strong and rigorous privacy notion, largely considered as the de-facto standard mechanism to protect sensitive users data in statistical data analysis with notable adoption by the US.~Census Bureau \cite{abowd2018us}, Google \cite{erlingsson2014rappor} and Apple \cite{cormode2018privacy}. 

While this framework has desirable properties its development has been focused on protecting the information contained in the training data, leaving thus possible exposure to the information being revealed during deployment by the users adopting the system. 
Further, to perform inference, each user is conventionally required to reveal the \emph{complete} set of features describing its data, even if they may not be \emph{all} essential to infer the intended prediction.
This not only poses severe privacy risks for the individuals using the learning systems but also requires companies and organizations massive human efforts to verify the correctness of the released information. 
Importantly, this setting may also violate the EU General Data Protection Regulation in the principle called \emph{data minimization}, which is cited as: ``Personal data shall be adequate, relevant and limited to what is necessary in relation to the purposes for which they are processed'' \cite{rastegarpanah2021auditing,regulation2016regulation}.

This paper challenges this setting and asks whether it is necessary to require \emph{all} input features for a model to return accurate or approximately accurate predictions at test time.  
We refer to this question as the \emph{redundant information leakage release for inference} problem. 

This unique question has profound implications for privacy in model personalization, where users are required to reveal large amounts of data. 
We show that, under a personalized setting, each individual may need to release only a small subset of their features to produce the \emph{same} prediction errors as those obtained when all features are available. 
Following this result, we also provide an efficient sequential algorithm that selects the smallest set of attributes to reveal by each individual. Evaluation over several learning tasks shows that individuals may be able to report as little as 10\% of their information to ensure the same level of accuracy of a model that uses the complete users' information.

\textbf{Contributions.} In summary, the paper makes the following contributions: {\bf (1)} it initiates a study to analyze which subset of data features should be released by each individual at deployment time, to induce a model having the same level of accuracy as if all features were released; 
{\bf (2)} it links this analysis to a new concept of {\emph{redundant information leakage}} and privacy, 
{\bf (3)} it proposes theoretically motivated and efficient algorithms that choose which attributes should be provided by each individual to minimize redundant information leakage, 
and 
{\bf (4)} it conducts a comprehensive evaluation illustrating that individuals may be able to report as little as 10\% of their information to ensure the same level of accuracy of a model that uses the complete users’ information.

To the best of our knowledge, this is the first work studying this connection between privacy and accuracy at test time.

\section{Related work}
\label{sec:related_work}

While we are not aware of studies on redundant information release for inference problems, we draw connections with differential privacy, feature selection, and active learning.

\textbf{Differential Privacy.}
Differential Privacy (DP) \cite{Dwork:06} is a strong privacy notion which determines and bounds the risk of disclosing sensitive information of individuals participating into a computation. 
In the context of machine learning, DP ensures that algorithms can learn the relations between data and predictions while preventing them from memorizing sensitive information about any specific individual in the training data. 
In such a context, DP is primarily adopted to protect training data \cite{abadi:16, JMLR:v12:chaudhuri11a,xie2018differentially} and thus the setting contrasts with that studied in this work, which focuses on identifying the superfluous features revealed by users at \emph{test time} to attain high accuracy. 
Furthermore, achieving tight constraints in differential privacy often comes at the cost of sacrificing accuracy, while the proposed privacy framework can reduce privacy loss without sacrificing accuracy under the assumption of linear classifiers.

\textbf{Feature selection.} 
Feature selection \cite{chandrashekar2014survey} is the process of identifying and selecting a relevant subset of features from a larger set for use in model construction, with the goal of improving performance by reducing complexity and dimensionality of the data. 
The problem studied in this work can be considered as a specialized form of feature selection with the added consideration of personalized levels, where each individual may use a different subset of features. This contrasts standard feature selection \cite{li2017feature}, which select the same subset of features for each data sample.
Additionally, and unlike traditional feature selection, which is performed during training and independent of the deployed classifier \cite{chandrashekar2014survey}, the proposed framework performs feature selection at deployment time and is inherently dependent on the deployed classifier.

\textbf{Active learning.}
Finally. the proposed framework shares similarities with active learning \cite{fu2013survey,settles2009active}, whose goal is to iteratively select samples for experts to label in order to construct an accurate classifier with the least number of labeled samples. Similarly, the proposed framework iteratively asks individuals to reveal one attribute given their released features so far, with the goal of minimizing the uncertainty in model predictions.

Despite these similarities, the proposed redundant information leakage concept is motivated by a privacy need and pertains to the analysis of features to release to induce the same level of accuracy as if all features were released. 

\section{Settings and Objectives}
\label{sec:settings}

We consider a dataset $D$ consisting of samples $(x, y)$ drawn from an unknown distribution $\Pi$. Here, $x$ is a feature vector with $x \in \mathcal{X}$, and $y \in \mathcal{Y} = [L]$ is a label with $L$ classes. The features in $x$ can be divided into two categories: \emph{public} $x_P$ and \emph{sensitive} features $x_S$. {The sets of public and sensitive features indexes in vector $x$ are represented as 
$P$ and $S$, respectively.
We consider classifiers $f_\theta : \mathcal{X} \to \mathcal{Y}$, which are trained on a public dataset from the same data distribution $\Pi$ above. The classifier produces a score over the classes, $\tilde{f}_\theta(x) \in \mathbb{R}^L$, and a final output class, $f_\theta(x) \in [L]$, given input $x$.
The model's outputs $f_\theta(x)$ and $\tilde{f}_\theta(x)$ are also often referred to as hard and soft predictions, respectively. 

Without loss of generality, we assume that all features in $\mathcal{X}$ are in the range of $[-1, 1]$. In this setting, we are given a trained model $f_\theta$ and, at prediction time, we have access to the public features $x_P$. These features may be revealed in response to a user query or may have been collected by the provider in a previous interaction. For the purpose of illustration, in the scope of the paper we consider the binary classification, where $L = \{0, 1\}$ and $\tilde{f}_{\theta} \in \mathbb{R}$. We refer to the Appendix for the multi-class settings where $L >2$.

In this paper, the term \emph{redundant information leakage} of a model, refers to the number of sensitive features that are revealed unnecessarily, meaning that their exclusion would not significantly impact the model's output. 
{\em Our goal is to design algorithms that accurately predict the output of the model using the smallest possible number of sensitive features, thus minimizing the data leakage at test time}. This objective reflects our desire for privacy.


\begin{figure*}
    \centering
    \includegraphics[width=0.3\textwidth,valign=c]{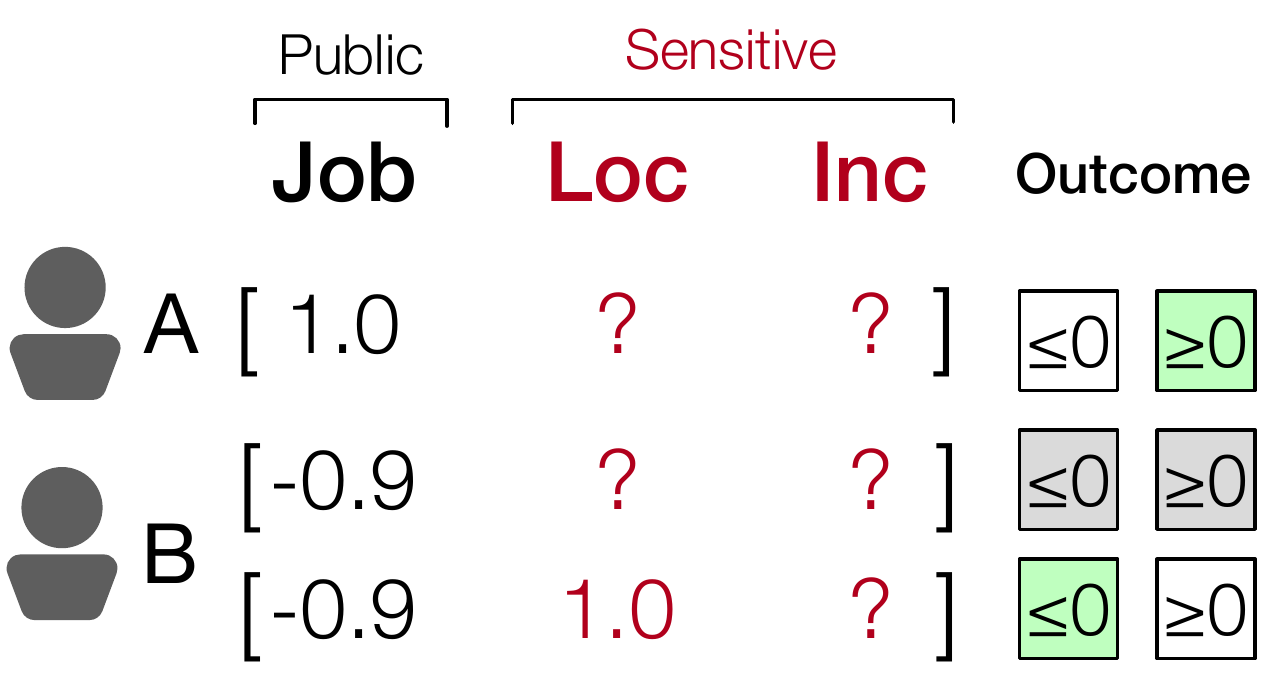}
    \includegraphics[width=0.4\textwidth,valign=c]{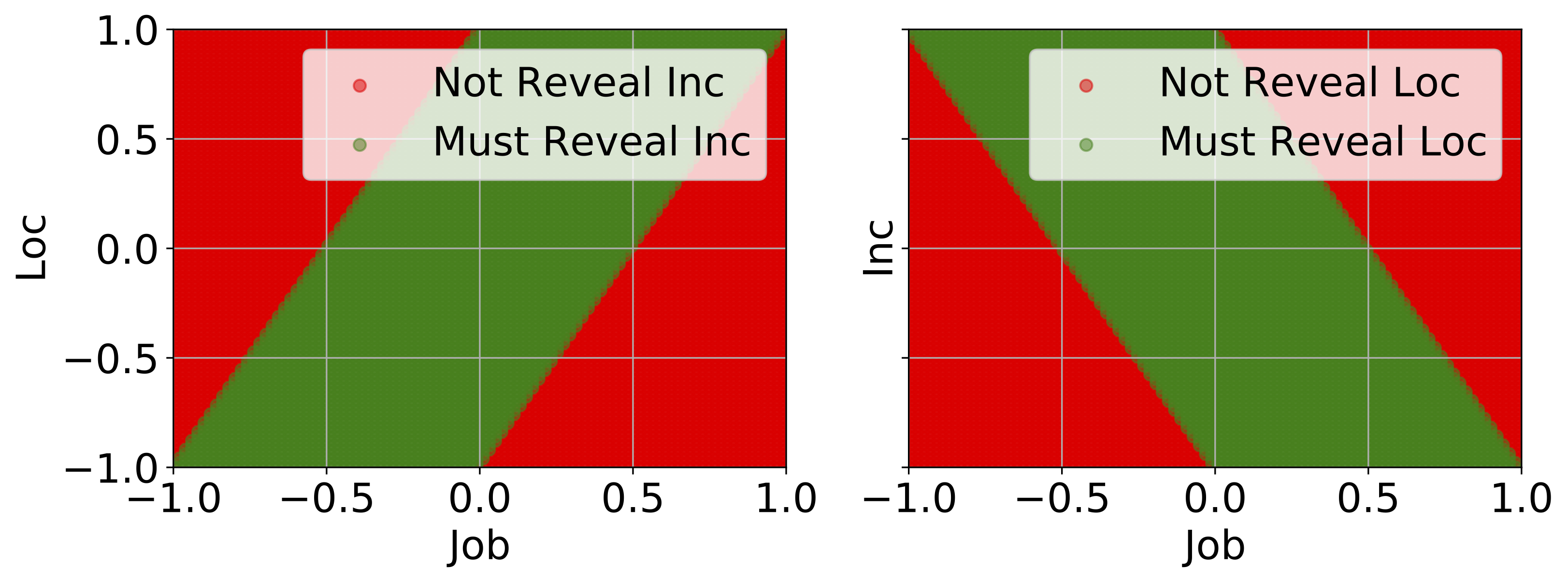} 
    \includegraphics[width=0.28\textwidth,height=80pt,valign=c]{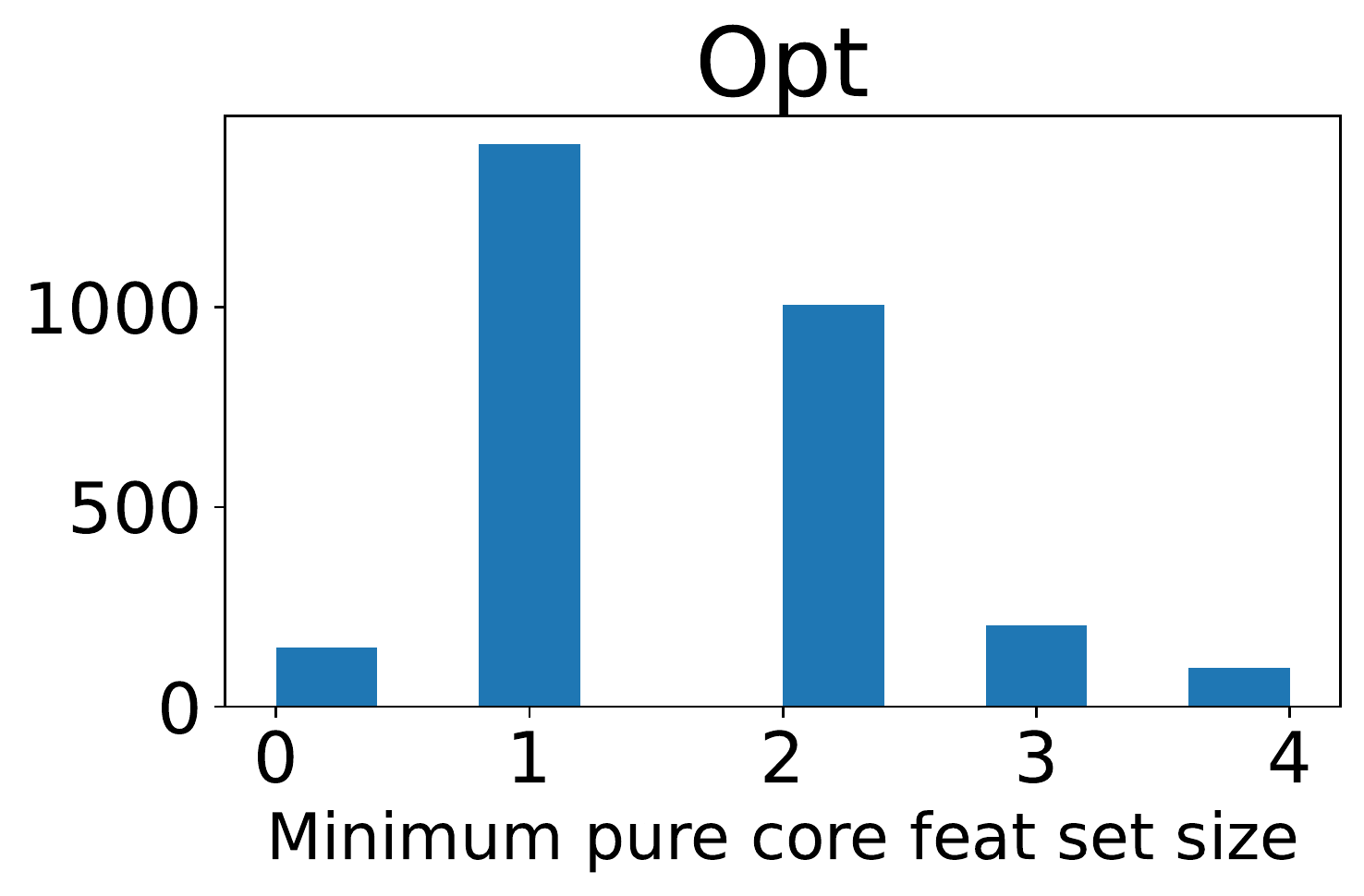}
    \caption{Left: Motivating example. Middle: Feature spaces illustrate the need for users to reveal their sensitive values based on their public values. Right: Frequency associated with the size of the \textbf{minimum} pure core feature set in the Credit card dataset under a logistic regression classifier.}
    \label{fig:motivation}
\end{figure*}

Before delving into the details of the paper, we provide an example to serve as motivation for several key points discussed throughout the document

Consider the illustration in Figure \ref{fig:motivation} (left). It exemplifies a loan approval task in which individual features are represented by the set $\{\textsl{Job}, \textsl{Loc(action)}, \textsl{Inc(ome)}\}$. The example assumes that the feature \textsl{Job} is the public feature $x_P$ while \textsl{Loc} and \textsl{Inc } are sensitive features $x_S$.
The example also considers a trained linear model 
\(f_\theta = 1.0\,\textsl{Job} - 0.5\,\textsl{Loc} + 0.5\,\textsl{Inc} \geq 0.\)
and looks at a scenario in which a user (A) has a public feature $\textsl{Job}=1.0$ and a user (B) has a public feature $\textsl{Job}=-0.9$.  Both users have sensitive feature values that are not known. However, notice how, for user A, the outcome can be determined with certainty even if they do not reveal any additional information; No matter the realizations for the sensitive features of A, their outcome will be unaltered, as all features are bounded in $[-1,1]$. For user B, in contrast, the outcome cannot be determined with certainty based on the public feature alone. But the release of sensitive feature $\textsl{Loc} = 1.0$ is sufficient to determine, with certainty, the classifier outcome. 

Figure \ref{fig:motivation} (middle) further illustrates the values of the sensitive features \textsl{Loc} and \textsl{Inc} in relation to the public feature \textsl{Job} which allows the classifiers' output to be determined without revealing additional information. 

This example highlights two important observations that motivate our study: (1) not all sensitive attributes may be required for decision-making at the time of inference, and (2) the number of relevant sensitive attributes that need to be revealed to make a decision may differ among individuals.

\section{Core Feature Sets} 

With these ideas in mind, this section introduces the concept of core 
feature set and its relationship with the uncertainty of the model predictions.
We discuss the main results in the paper and report all proofs in 
Appendix \ref{app:proofs}. 

Throughout the paper, we use $R$ and $U$ to denote, respectively, the set of all revealed and unrevealed features indices of the sensitive features $S$. 
Given a vector $x$ and an index set $I$, we denote $x_I$ as the vector of entries indexed by $I$ and $X_I$ as the associated random variable. 
Finally, we write $f_{\theta}(X_U, X_R\!=\!x_R)$ as a shorthand for $f_\theta(X_U, X_R\!=\!x_R, X_P\!=\!x_P)$ to denote the prediction made by the model when the features in $U$ are unrevealed. 

We aim to create algorithms that can identify the smallest set of sensitive attributes to reveal to render the model's output certain (with high probability) regardless of the unrevealed attributes' values. 
Such a set is denoted \emph{core feature set}.

\begin{definition}[Core feature set]
\label{def:core_set}
Consider a subset $R$ of sensitive features $S$, and let $U \!=\! S \setminus R$ be the unrevealed features. The set $R$ is a core feature set if, for some $\tilde y \in \cY$,
\begin{equation}\label{eq:cfs}
    \Pr\big( f_\theta(X_{U}, X_{R} = x_R) = \tilde y \big) \geq 1-\delta,
\end{equation}
where $\delta \in [0,1]$ is a failure probability. 
\end{definition}
When $\delta = 0$ the core feature set is called \textbf{pure}. 
Additionally, the label $\tilde y$ satisfying Equation \eqref{eq:cfs} 
is called the \emph{representative label} for the core feature set $R$. 
The concept of the representative label $\tilde y$ is crucial for the 
algorithms that will be discussed later. These algorithms use limited 
information to make predictions. When predictions are made using a set 
of unrevealed features, the representative label $\tilde{y}$ will be 
used in place of the model's prediction.

The following is a useful property of core feature sets used by this 
work to minimize redundant information leakage.
\begin{proposition}
    \label{thm:delta_vs_entropy}
    Let $R \subseteq S$ be a core feature set with failure probability $\delta <0.5$. 
    Then, there exists a monotonic decreasing function $\epsilon:\mathbb{R}\to\mathbb{R}$ with $\epsilon(0)=0$ such that: 
  \[
      H\big[ f_\theta( X_U, X_R = x_R ) \big] \leq \epsilon(\delta),
  \]
where $H[Z] \text{=} -\sum_{z \in [L]}\Pr(Z=z) \log \Pr(Z=z)$ 
is the entropy of the random variable $Z$. 
\end{proposition}

This property highlights the relationship between core feature sets 
and entropy associated with the model that uses incomplete information. 
Smaller $\delta$ values result in less uncertainty in the model's 
predictions and when $\delta$ is equal to zero (or when $R$ is a pure 
core feature set), we have complete knowledge of the model's predictions 
even without observing $x_{U}$.
Thus this property also illustrates the relationship between the 
failure probability $\delta$ and the uncertainty of model predictions. 

It is worth noticing that more accurate predictions also require 
revealing more information, as highlighted in the previous result and 
the following celebrated information theoretical result.
\begin{proposition}
    \label{thm:cond_entropy}
    Given two subsets $R$ and $R'$ of sensitive features $S$, with 
    $R \subseteq R'$, 
    \[
      H\big( f_\theta(X_U, X_R=x_R) \big) \geq 
      H\big( f_\theta(X_{U'}, X_{R'}=x_{R'}) \big),
    \]
    where $U=S\setminus R$ and $U'=S\setminus R'$.
\end{proposition}
Thus, the parameter $\delta$ plays an important role in balancing the trade-off between the \emph{privacy loss} and the \emph{model performance}. 
It controls how much sensitive information needs to be revealed to make accurate predictions (for a desired level of uncertainty in the model's predictions). 
As $\delta$ gets larger, less sensitive features need to be revealed, leading to smaller information leakage but also less accurate model predictions, and vice-versa. 


Note that, as pointed out in the previous example, the core feature set is not unique for all users. This is also highlighted in Figure \ref{fig:motivation} (right), which illustrates the minimum pure core feature sets computed using a logistic regression classifier on the Credit dataset \cite{UCIdatasets}. The figure shows that many individuals need to release \emph{no} additional information to obtain the model predictions and that most individuals can get accurate model predictions with certainty by releasing just $\leq 2$ sensitive features. 
These connections, together with the previous observations linking core feature sets to entropy motivate the proposed online algorithm.

\section{Personalized feature release (PFR)}
\label{sec:PFR}
The goal of the proposed algorithm, called Personalized feature release 
(PFR), is to reveal sensitive features one at a time based on their 
\emph{released} feature values. This section provides a high-level 
description of the algorithm and outlines the challenges in some of 
its aspects. Next, Section \ref{sec:PFR_linear}, applies PFR to 
linear classifiers and discusses its performance on several datasets 
and benchmarks. Further, Section \ref{sec:PFR_nonlinear}, extends PFR 
to non-linear classifiers and considers an evaluation over a range of 
standard datasets.
In the subsequent sections, we assume that the input features are 
jointly distributed as Gaussians with mean vector $\mu$ 
and covariance matrix $\Sigma$, 
unless stated otherwise. Additionally, as our motivation suggests, we 
will concentrate solely on maintaining privacy at deployment time.

\paragraph{High-level ideas of PFR.}
At a high level, the algorithm chooses a feature to reveal by inspecting 
the posterior probabilities $\Pr(X_j | X_R = x_R, X_P = x_P)$ for each unrevealed feature $j \in U$ and with respect to the revealed sensitive features $x_R$ and the public features $x_P$.
Given the current set of features revealed $x_R$ and unrevealed $x_U$, the algorithm chooses the next feature $j \in U$ such that:
\begin{align}
  j &= \argmax_{j \in U} F(x_R, x_j; \theta) \notag\\
  \label{eq:scoring}
    &= \argmax_{j \in U} -H\big[ f_\theta(X_j=x_j, X_{U\setminus \{j\}}, X_R=x_R)\big],
\end{align}
where $F$ is a \emph{scoring function} that measures how much information can be gained on the model's predictions if feature $X_j$ is revealed. 
Upon revealing feature $X_j$ with a value of $x_j$, the algorithm adjusts the posterior probabilities for all remaining unrevealed features. The process concludes when either all sensitive features have been disclosed or a core feature set has been identified. 

The remainder of the section delves into the difficulties of calculating the scoring function $F$, including the unknown value of $X_j$ beforehand and methods for determining if a set of revealed features constitutes a core feature set.

\subsection{Computing the scoring function $F$}
The scoring function $F$ quantifies the level of certainty in model 
predictions when a user reveals the value of feature $X_j$. 
There are two challenges to consider. First, the value of $X_j$ is 
unknown until the decision is made, challenging the computation of the 
entropy function. Second, even if the value of $X_j$ were known, 
determining the entropy of model predictions in an efficient manner  
is a further difficulty.
We next discuss how to overcome these challenges.

To address the first challenge, we exploit the information encoded in 
the revealed features to infer $x_j$. 
Thus, we can compute the posterior probability $\Pr(X_j | X_R \!=\! x_R)$ 
of the unrevealed feature $X_j$ given the values of the revealed ones. 
This estimate allows us to modify the scoring function, abbreviated as 
$F(X_j)$, to be the expected negative entropy given the randomness of $X_j$.
\begin{subequations}
\label{eq:exp_entropy}
\begin{align}
    F(X_j) &= 
    \mathbb{E}_{X_j} - \big[H[f_\theta(X_j, X_{U\setminus \{j\}}, X_R \!=\! x_R)\big] \notag\\
    &= - \int {H\big[f_\theta (X_j\!=\!z, X_{U\setminus \{j\}}, X_R\!=\!x_R)\big]}
    \label{eq:exp_entropy1}\\
    &\hspace{30pt} \times 
    {\Pr(X_j \!=\! z | X_R \!=\! x_R) dz},
    \label{eq:exp_entropy2}
\end{align}
\end{subequations}
where $z \in {\cal X}_j$ is a value in the support of $X_j$. 

Estimating this scoring function efficiently is however challenged by 
the presence of two key components. The first (Equation \eqref{eq:exp_entropy1}) 
is the entropy of the model's prediction given a specific unrevealed 
feature value, $X_j = z$. This prediction is a function of the random 
variable $X_{U \setminus \{j\}}$, and, due to Proposition \ref{thm:delta_vs_entropy}, 
its estimation is related to the conditional densities $\Pr(X_{U \setminus \{j\}} \vert X_R = x_R, X_j = z)$.
The second component (Equation \eqref{eq:exp_entropy2}) is the conditional probability 
$Pr(X_j = z| X_R = x_R)$. Computing these conditional densities efficiently is discussed next.

The following result relies on the joint Gaussian assumption of the input features and will be useful in providing a computationally efficient method to estimate such conditional density functions. In the following, $\Sigma_{IJ}$ represents a sub-matrix of size $|I| \times |J|$ of a matrix $\Sigma$ formed by selecting rows indexed by $I$ and columns indexed by $J$. 

\begin{proposition}
\label{prop:2} The conditional distribution of any subset of unrevealed features 
$U' \in U$, given the the values of released features $X_R =x_R$ is given by:
\begin{align*}
\Pr(X_{U'} | X_R = x_R)  &= 
    \mathcal{N}\bigg(\mu_{U'}  + \Sigma_{U', R} \Sigma^{-1}_{R, R} 
        (x_R - \mu_R),\\ 
        & \Sigma_{U',U'} - \Sigma_{U',R}\Sigma^{-1}_{R,R} \Sigma_{R,U'} \bigg),
\end{align*}
where $\Sigma$ is the covariance matrix 
 \end{proposition}
 
To complete Equation \ref{eq:exp_entropy}, we must estimate the entropy 
$H[f_\theta (X_j = z, X_{U\setminus \{j\}}, X_R=x_R)]$ for a specific 
instance $z$, drawn from the distribution $\Pr(X_j | X_R = x_R)$ 
(see Equation \eqref{eq:exp_entropy1}). This can be achieved by 
estimating $\Pr(\tilde{f}_{\theta}(X_j = z, X_{U\setminus\{j\}}, X_R = x_R))$, as $f_\theta = \bm{1} \{\tilde{f}_\theta \geq 0\}$, where $\bm{1}$ is the indicator function and 
in the following sections, we will show how to assess this estimate 
for linear and non-linear classifiers. 
Finally, by approximating the distribution over soft model predictions through Monte Carlo sampling, 
we can compute the score function in $F(X_j)$, as
\begin{align}
     F(X_j) &= \mathbb{E}_{X_j} - \left[ H \left[f_\theta (X_j, X_{U\setminus\{j\}}, X_R =x_R)\right] \right] \label{eq:monte_carlo}\\
     &\approx  -\sum_{z' \in {\bm Z}} 
     H\left[ f_\theta (X_j=z', X_{U\setminus\{j\}}, X_R =x_R) \right],\notag   
\end{align}
where $\bm Z$ is a set of random samples drawn from $\Pr(X_j | X_R = x_R)$ and estimated through Proposition \ref{prop:2}, which thus can be computed efficiently.

\subsection{Testing a core feature set}

As reviewed above, the proposed iterative algorithm stops when it determines 
whether a subset $R$ of the sensitive feature set $S$ is a core feature 
set. We divide this verification process into two cases:
When $\delta=0$, verifying that $R$ is a pure core feature set only 
requires checking if {$f_{\theta}(X_U, X_R =x_R)$} is constant for all 
realizations of $X_U$. We demonstrate, in Section 
\ref{sec:PFR_linear}, that this can be accomplished in linear 
time for linear classifiers without any input distribution assumptions. 
When $\delta>0$, such a property is no longer valid. Recall that, in 
order to verify a core feature set as per Definition \ref{def:core_set}, 
we need to estimate the distribution of $\Pr(\tilde{f}_{\theta}(X_U, X_R =x_R))$. 
In Section \ref{sec:PFR_linear}, we show that one can analytically estimate 
this distribution for linear classifiers, while in Section \ref{sec:PFR_nonlinear} 
we show how to approximate this distribution locally, and use this estimate 
to derive a simple, yet effective (in practice), estimator.

\section{PFR for linear classifiers}
\label{sec:PFR_linear}

This section will  devote to estimating the distribution 
$\Pr(\tilde{f}_{\theta}(X_j \!=\! z, X_{U\setminus\{j\}}, X_R \!=\! x_R))$, or 
simply expressed as $\Pr(\tilde{f}_{\theta}(X_U, X_R \!=\! x_R))$  and provides an instantiation 
of the PFR algorithm for linear classifiers. 
In particular, it shows that when the input features are jointly Gaussian, 
both the estimation of the conditional distributions required to compute 
the scoring function $F(X_j)$ and the termination condition to test whether 
a set of revealed features is a core feature set, can be computed 
efficiently. This is an important property for the developed algorithms, 
which are considered online and interactive protocols.

\label{sec:PFR_linear}
\subsection{Efficiently Estimating 
$\Pr(\tilde{f}_{\theta}(X_{U}, X_R = x_R))$}
For a linear classifier $\tilde{f}_{\theta} = \theta^\top x$, and under 
the Gaussian distribution assumption adopted, the model predictions 
$\tilde{f}_{\theta}(x)$ are also Gaussian, as highlighted by the 
following result.
\begin{proposition}
\label{prop:3}
The model predictions before thresholding, $\tilde{f}_{\theta}(X_U, X_R = x_R) =\theta_U X_U + \theta_R x_R $ 
is a random variable with a Gaussian distribution $\mathcal{N}\big(m_f, {\sigma_f^2} \big)$, 
where
\begin{align}
    m_f &= \theta_R x_R + \theta^\top_U \big(  \mu_U + \Sigma_{U, R} \Sigma^{-1}_{R,R} ( x_{R} - \mu_R) \big) \label{eq:prop4_a}\\ 
    \sigma^2_f &= \theta^\top_{U} \big( \Sigma_{U,U} - \Sigma_{U, R} \Sigma^{-1}_{R,R} \Sigma_{R,U} \big) \theta_{U}, \label{eq:prop4_b}
\end{align}
where $\theta_U$ is the sub-vector of parameters $\theta$ corresponding
to the unrevealed features $U$. 
 \end{proposition} 

The above result is used to assist in calculating the conditional 
distribution of model predictions $f_\theta( x)$, following thresholding. 
This is a random variable that adheres to a Gaussian distribution, as shown 
next, and will be used to compute the entropy of the model predictions, 
as well as to determine if a subset of features constitutes a core set.
 
\begin{proposition}
\label{prop:4}
Let the model predictions prior thresholding $\tilde{f}_\theta(X_U, X_R =x_R)$ 
be a random variable following a Gaussian distribution $\mathcal{N} (m_f, 
\sigma^2_f)$. Then, the model prediction following thresholding $f_{\theta}( X_U, X_R =x_R)$ 
is a random variable following a Bernoulli distribution $Bern(p)$ 
with $p = \Phi( \frac{m_f}{\sigma_f})$, where $\Phi(\cdot)$ refers to 
the  CDF of the standard normal distribution, and $m_f$ and $\sigma_f$,
are given in Equations \eqref{eq:prop4_a} and \eqref{eq:prop4_b}, respectively.
\end{proposition}

\begin{figure*}[tb]
\captionsetup{justification=centering}
\centering
\includegraphics[width = 1.0\linewidth]{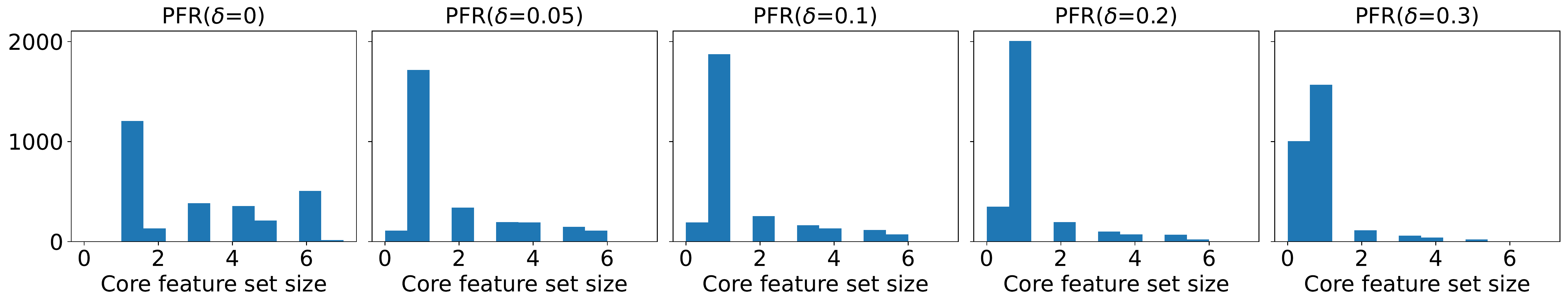}
\caption{Histogram of core feature set size  for PFR under different $\delta$ on Bank  dataset when $|S|=7$ and the underlying classifier is Logistic Regression}
\label{fig:hist_linear}
\end{figure*}

\begin{algorithm}[t]
  \caption{PFR for linear classifiers}
  \label{alg:alg1}
  \setcounter{AlgoLine}{0}
  \SetKwInOut{Input}{input}
  \SetKwInOut{Output}{output}
  \SetKwRepeat{Do}{do}{while}
  \Input{A test sample $x$; Training data $D$} 
  \Output{A core feature set $R$  and its representative label $\tilde y$\!\!\!\!\!\!\!}
  $\mu \gets \frac{1}{|D|} \sum_{(x,y) \in D} x$\\
  $\Sigma  \gets \frac{1}{|D|} \sum_{(x,y) \in D} (x -\mu) (x - \mu)^\top$\\ 
  Initialize $R = \emptyset$\\
  \While{True} {
    \eIf{$R$ is a core feature set with repr.~label $\tilde y$}{
      \Return{$(R, \tilde y)$}
    }{
    \ForEach{$j \in U$} {
       Compute $\Pr(X_j | X_R = x_R)$  (using Prop.~\ref{prop:2})\\
      $\bm Z \gets \text{sample}(\Pr(X_j |X_R = x_R))$ T times\\
       Compute $\Pr\left(f_{\theta}(X_j = z, X_{U\setminus \{j\}}X_R = x_R)\right)\!\!\!\!\!\!\!$ (using Prop.~\ref{prop:3} and \ref{prop:4})\\
       Compute $F(X_j)$ (using Eq. \eqref{eq:monte_carlo})
       }
    }
     $j^* \gets \argmax_j F(X_j)$\\
     $(R, U) \gets R \cup \{j^*\},\; U \setminus \{j^*\}$
    }
\end{algorithm}

\subsection{Testing pure core feature sets}
\label{sec:pure_core_set_linear}
In this subsection, we outline the methods for determining if a subset $U$ is a pure core feature set, and, if so, identifying its representative label. As per Definition \ref{def:core_set}, $U$ is a pure core feature set if $f_{\theta}(X_U, X_R =x_R) = \tilde{y}$ for all $X_U$. Equivalently, $\tilde{f}_{\theta}(X_U, X_R = x_R) = \theta^\top_U X_U + \theta^\top_R x_R$ must have the same sign (either positive or negative) for all $X_U$ in the range of $[-1, 1]^{|U|}$. Rather than evaluating all possible values of $\tilde{f}_{\theta}(X_U, X_R = x_R)$, we only need to examine if the maximum and minimum values have the same sign. By virtue of the linear programming property under the box constraint $X_U \in [-1, 1]^{|U|}$, it follows that:
\begin{align}
\begin{split}
\max_{X_U} \,  & \theta^\top_U X_U + \theta^\top_R x_R = \|\theta\|_1 + \theta^\top_R x_R \\
\min_{X_U} \, & \theta^\top_U X_U + \theta^\top_R x_R = - \|\theta\|_1 + \theta^\top_R x_R.
\label{eq:min_max_linear}
\end{split}
\end{align}
Therefore, if the sum $\|\theta\|_1 + \theta^\top_R x_R$ and the difference $-\|\theta\|_1 + \theta^\top_R x_R$ are both negative (non-negative), then $U$ is considered a pure core feature set with representative label $\tilde{y}=0$ ($\tilde{y}=1$), otherwise $U$ is not a pure core feature set.

Importantly, determining whether a subset $R$ of sensitive features $S$ constitutes a pure core feature set can be accomplished in linear time with respect to the  number of features. 
\begin{proposition}
\label{thm:core_linear_verf}
Assume $f_\theta$ is a linear classifier. Then, determining if a 
subset $U$ of sensitive features $S$ is a \emph{pure} core feature set 
can be performed in $O(|P| + |S|)$ time.
\end{proposition}

\begin{figure*}[t]
\centering
\includegraphics[width = 0.45\linewidth]{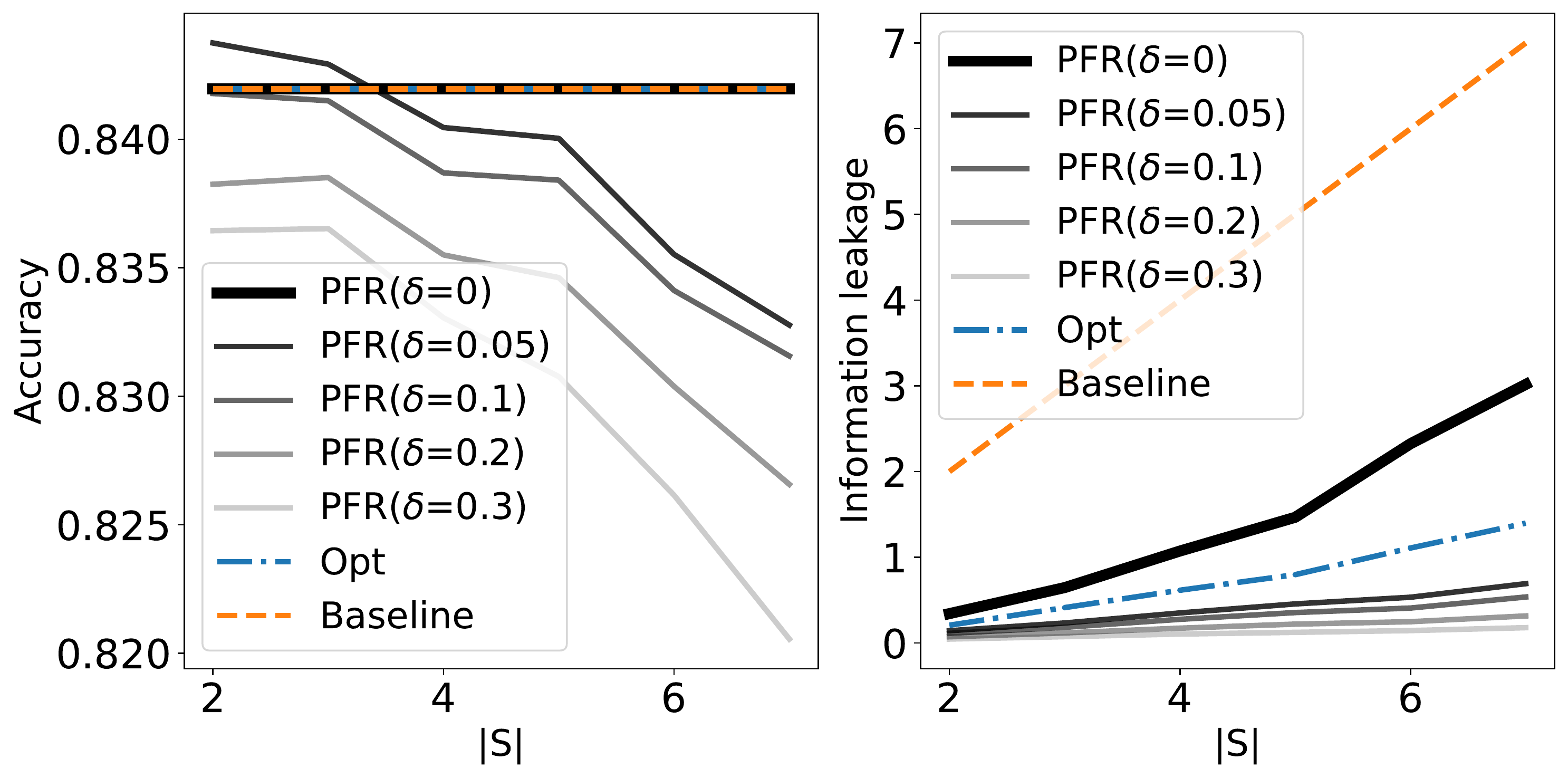}
\includegraphics[width = 0.45\linewidth]{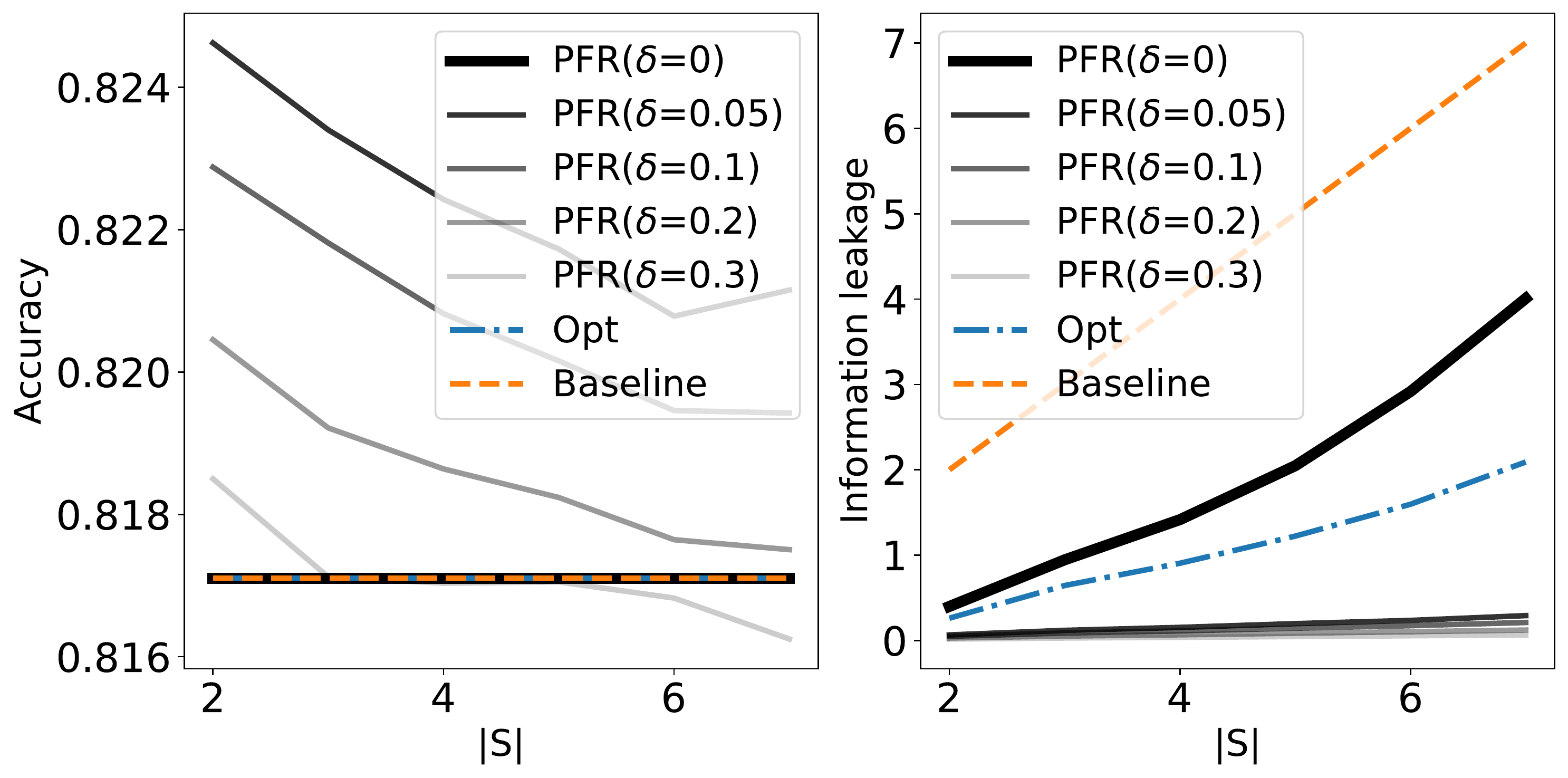}
\caption{Accuracy and redundant information leakage for different choices of number of sensitive features $|S|$ on Insurance (left) and Credit (right) datasets using a Logistic Regression classifier.}
\label{fig:linear}
\end{figure*}

\subsection{PFR-linear Algorithm and Evaluation}
\label{sec:PFRlinear_eval}

A pseudo-code of PFR specialized for linear classifiers is reported in 
Algorithm \ref{alg:alg1}. 
The algorithm takes as input a sample $x$ (which only exposes the set 
of public features $x_P$) and uses the training data $D$ to estimate 
the mean and covariance matrix needed to compute the conditional distribution
of the model predictions given the unrevealed features (lines 1 and 2), 
as discussed above. 
After initializing empty the set of revealed features to the (line 3)
it iteratively releases one feature until a core feature set 
(and its associated representative label) are determined (line 5), as
discussed in detail in Section \ref{sec:pure_core_set_linear}.
The released feature $X_{j^*}$ is the one, among the unrevealed features 
$U$, that maximizes the scoring function $F$ (line 13). Computing such 
a scoring function requires estimating the conditional distribution 
$\Pr(X_j | X_R = x_R)$ (line 9), constructing a sample set $\bm Z$ 
from such distribution (line 10), and approximating the distribution 
over soft model predictions through Monte Carlo sampling to compute (line 11).
Finally, the algorithm updates the set of the revealed and unrevealed 
features (line 14). 

Notice that PFR relies on estimating the mean vector and covariance matrix 
from the training data, which is considered public, for the scope of this paper. 
If the training data is private, various techniques exist to release DP mean, 
and variance \cite{liu2021robust,amin2019differentially} and can be readily adopted. 
However, the protection of training data is beyond the scope of this work.

\textbf{Evaluation.} Next, this section evaluates the effectiveness of PFR in minimizing information leakage. The experiments are conducted on six standard UCI datasets \cite{UCIdatasets}. 
We discuss here a selection of these results and refer the reader to the Appendix for additional 
experiments.

Figure \ref{fig:hist_linear} reports the snapshot on the redundant data
leakage subject by various users on a Logistic regression classifier
trained on the Bank dataset \cite{UCIdatasets} (more details reported 
in the Appendix), when using the proposed PFR algorithm for various 
core feature set failure probability $\delta$ levels.
The benefits of PFR are clearly evident from this histogram. For each 
testing sample, PFR finds core feature sets that are much smaller than the overall 
sensitive feature set size $|S|=7$. Additionally, notice that when $\delta >0$, it finds 
core features sets of size smaller than 2 for the vast majority of the individuals. 
{\em This suggests that a significant number of users would need to disclose 
only a small fraction of all of their sensitive information to allow 
the model to make accurate predictions either with complete certainty 
or with very high confidence.}






To further illustrate the advantages of PFR, we compare it to a \emph{baseline} and an \emph{optimal} model for various choices of the number of sensitive attributes $|S| \in [2,7]$. 
The \emph{baseline}, in this context, refers to the use of the original classifier, which requires users to disclose all sensitive features.
The \emph{optimal} model refers to the process of using a brute force method to identify the minimum core feature set and its representative label by evaluating all possible subsets of sensitive features. Once identified, this representative label is used as the model prediction when not all sensitive features are disclosed. Verification tests are used to determine if a subset is a core feature set. It is important to note that this method is not only computationally inefficient due to the exponential number of cases, but also infeasible to implement in practice as it assumes that all sensitive features are known.

For each choice of $|S|$, we randomly select $|S|$ features from the entire set of features and designate them as sensitive attributes. The remaining features are considered as public attributes. The average accuracy and information leakage are then reported based on 100 random selections of the sensitive attributes. Additional details on the experimental settings can be found in Appendix Section \ref{sec:app_exp}. 

The performance results in terms of accuracy (left subplots) and information leakage (right subplots) are presented in Figure \ref{fig:linear}. It is observed that across all datasets, PFR with $\delta =0$ are able to identify a pure core feature set that is much smaller than the set of sensitive features. As a result, only a small percentage of sensitive features need to be disclosed by users, while maintaining the same level of accuracy. Furthermore, PFR with $\delta=0$, identifying pure core feature sets, can retain the same accuracy as the Baseline models. This implies that under linear models, privacy (as defined in this paper) can be achieved ``for free''!. Additionally, 
notice that how $\delta$ increases, fewer features need to be revealed by users, but at the cost of a decrease in accuracy, generally. Notice also that there may be cases (e.g., right subplots) where such features do not correlate well with the predictions, and not revealing them may thus even improve the prediction accuracy (this aspect is related to feature selection). Generally, however, the larger the failure probability $\delta$ the more information leakage can be protected but at a cost of a larger drop in accuracy. At the same time, notice how marginal is the decrease in accuracy, which demonstrates the robustness of the proposed model.

\section{PFR for non-linear classifiers}
\label{sec:PFR_nonlinear}

Next, the paper focuses on computing the estimate $\Pr(\tilde{f}_\theta(X_U , X_R = x_R))$ 
and determining core feature sets when $f_\theta$ is a nonlinear classifier. 
Then, the section presents results that illustrate the practical benefits of 
PFR in minimizing information leakage on neural networks. 
The determination of core feature sets relies on the assumption that the classifiers are $\Delta$-robust, i.e., $\forall x, x' \in \cX, \mbox{s.t:} \ \| x-x'\|_{\infty} \leq \Delta$ then $f_{\theta}(x) = f_{\theta}(x')$. In practice, however, we show that, even in the presence of arbitrary classifiers, the proposed PFR is able to significantly reduce information leakage at test time.

\begin{figure*}[t]
\centering
\includegraphics[width = 0.45\linewidth]{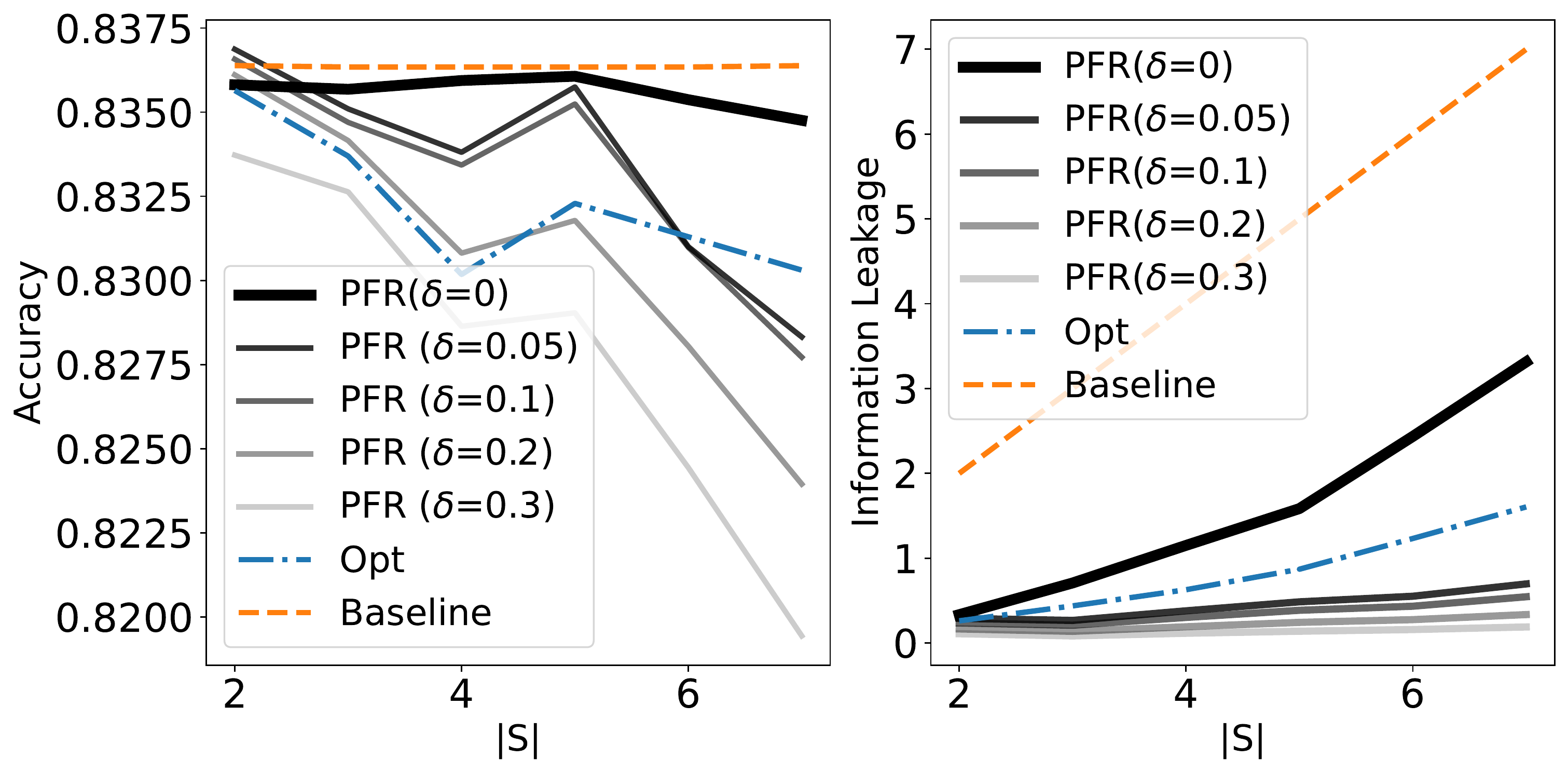}
\includegraphics[width = 0.45\linewidth]{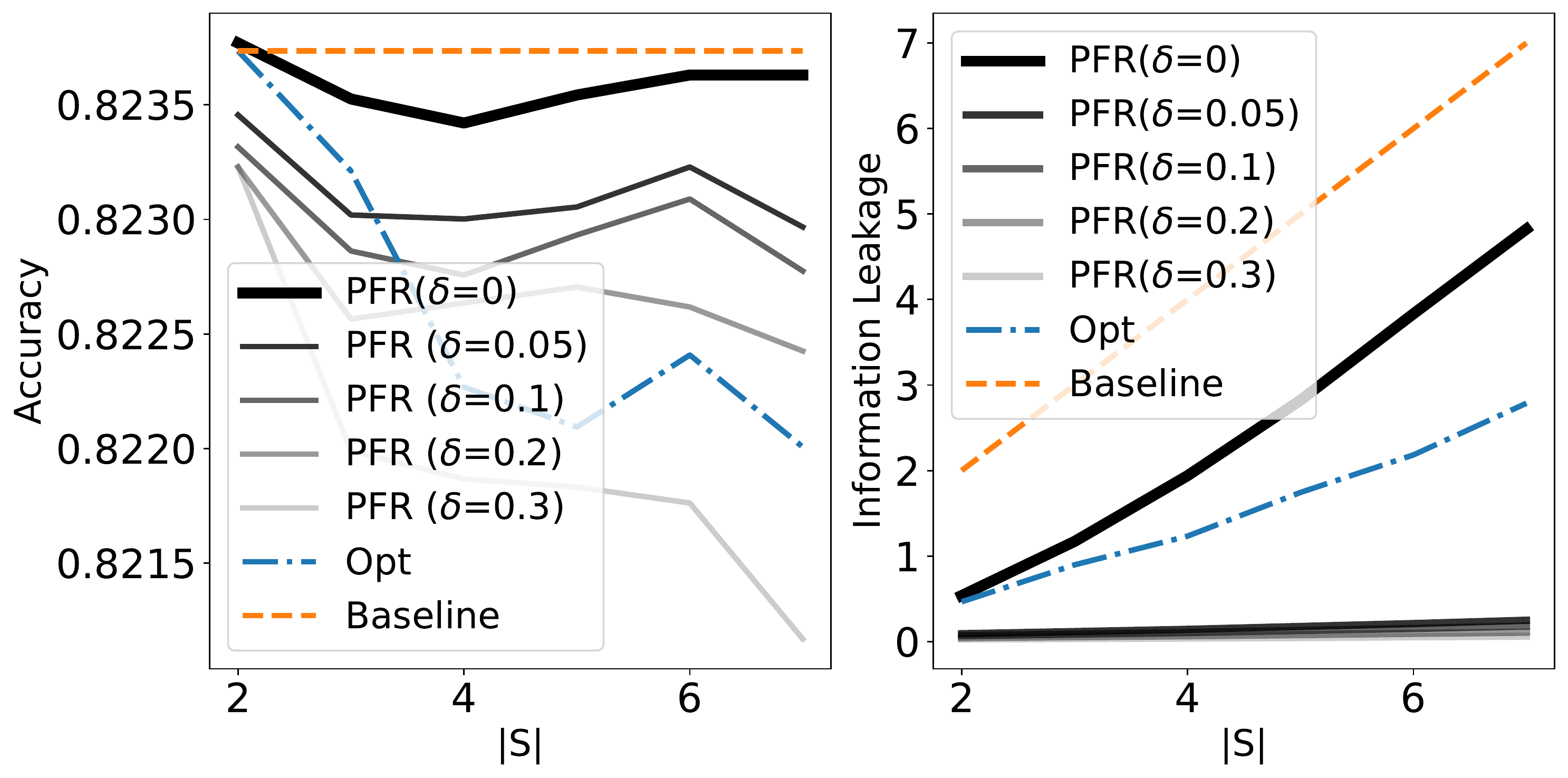}
\caption{Accuracy and redundant information leakage for different choices of number of sensitive features $|S|$ on Insurance (left) and Credit (right) datasets using a nonlinear (neural network) classifier.}
\label{fig:compare_nonlinear}
\end{figure*}

\subsection{Efficiently estimating $Pr(\tilde{f}_{\theta}(X_U, X_R =x_R))$}
First notice that even if the input features $x$ are jointly Gaussian, the outputs $f_\theta(x)$ of the classifier will no longer follow a Gaussian distribution after undergoing a non-linear transform. This makes estimating the distribution of $\Pr(\tilde{f}_{\theta}(X_U, X_R =x_R)$ more challenging. To address this issue, the paper proposes to locally approximate the model predictions $\tilde{f}\theta$ using a Gaussian distribution. This approach is demonstrated in the following result.
\begin{theorem}
\label{thm:taylor_approx}
The distribution of the  random variable $\tilde{f}_\theta = \tilde{f}_{\theta}(X_U, X_R =x_R)$ where $X_U \sim \mathcal{N}\big(\mu^{\mbox{pos}}_{U}, \Sigma^{\mbox{pos}}_{U} \big)$ can be  approximated by a Normal distribution as
\begin{align}
    \tilde{f}_\theta \sim \mathcal{N} \big( \tilde{f}_{\theta}( X_U = \mu^{pos}_U, X_R = x_R) , g^\top_{U}\Sigma^{pos}_{U}g_U\big) 
    \label{eq:taylor_arppox}
\end{align}
where $g_{U} = \nabla_{X_U}\tilde{f}_{\theta}(  X_U = \mu^{pos}_{U}, X_R = x_R)$ is the gradient of model prediction at $X_U = \mu^{pos}_{U}$. 
\end{theorem}
In the above, the mean vector $\mu^{\mbox{pos}}_{U}$ and and covariance matrix $\Sigma^{\mbox{pos}}_{U}$ of $Pr(X_U |X_R =x_R)$ are obtained from Proposition \ref{prop:2}. The result above relies on a first-order  Taylor approximation of the classifier $f_\theta$ around its mean.


\subsection{Testing pure core feature sets}
To determine if a subset $U$ of the sensitive features $S$ is a pure core feature set, we consider a set of $(\frac{1}{\Delta})^{|U|}$ input points, represented by $Q = {[X_U, x_R] }$. The entries corresponding to the revealed features are fixed with the value $x_R$, while the entries corresponding to the unrevealed features are evenly spaced over the cube $[-1, 1]^{|U|}$. The test verifies if the model predictions $f_{\theta}(x)$ remain constant for all $x$ in $Q$. 
Note that the computational runtime of this verification process is affected by the degree of robustness $\Delta$ of the underlying classifier $f$. Rendering such a procedure more generally computationally efficient will be an interesting direction for future work. In the next section, we will show that even considering arbitrary classifiers (e.g., we use standard neural networks), PFR can reduce information leakage dramatically when compared to standard approaches. 

\subsection{PFR-nonlinear Algorithm and Evaluation}
The FPR algorithm for non-linear classifiers differs from Algorithm \ref{alg:alg1} only in the method of calculating the estimates for the distribution of the soft model predictions, represented by $\Pr(f_\theta(X_j = z, X_{U\setminus {j}}, X_R=x_R))$ (line 11), by utilizing the results in Theorem \ref{thm:taylor_approx} and Proposition \ref{prop:4}. Additionally, the algorithm's termination test relies on the discussion presented in the previous section. A complete description of the algorithm is reported in Appendix \ref{app:pseudocodes}.

\textbf{Evaluation.}
Next, we assess the performance of PFR in reducing information leakage when standard non-linear classifiers are adopted. Specifically, we use a neural network with two hidden layers and ReLU activation functions as baselines classifiers and train models using stochastic gradient descent (as specified in more detail in Appendix \ref{sec:app_exp}). The evaluation, baselines, and benchmarks adopted follow the same settings as those adopted in Section \ref{sec:PFRlinear_eval}.
 
Figure \ref{fig:compare_nonlinear} illustrates the results in terms of accuracy (left subplots) and information leakage (right subplots). 
Unlike linear classifiers, non-linear models using PFR with a failure probability $\delta=0$ cannot ensure the same level of accuracy as the baseline models. However, notice how small this difference in accuracy is. Remarkably, a failure probability $\delta=0$ allows users to release less than a half and up to 90\% less sensitive features across different datasets while obtaining accuracies comparable to those of traditional classifiers. 
Notice also that when more relaxed failure probabilities are considered the information leakage reduces significantly. For example, when $\delta = 0.05$,  users require to release only 5\% of their sensitive features while retaining comparable accuracies to the baseline models (the largest accuracy difference reported was 0.005\%). 
{\em These results are significant: They demonstrate that the introduced privacy leakage notion and the proposed algorithm can become an important tool to safeguard the privacy of individual's data at test time, without excessively compromising accuracy}.

\section{Conclusion}
This paper introduced the concept of information leakage at test time whose goal is to minimize the number of features that individuals need to disclose during model inference while maintaining accurate predictions from the model. The motivations of this notion are grounded in the privacy risks imposed by the adoption of learning models in consequential domains, by the significant efforts required by organizations to verify the accuracy of the released information, and align with the data minimization principle outlined in the GDPR. 
The paper then discusses an iterative and personalized algorithm that selects the features each individual should release with the goal of minimizing information leakage while retaining exact (in the case of linear classifiers) or high (for non-linear classifiers) accuracy. 
Experiments over a range of benchmarks and datasets indicate that individuals may be able to release as little as 10\% of their information without compromising the accuracy of the model, providing a strong argument for the effectiveness of this approach in protecting privacy while preserving the accuracy of the model.

\section*{Acknowledgements}
This research is partially supported by NSF grant 2133169, NSF CAREER 
Award 2143706. Fioretto is also supported by a Google Research Scholar 
Award and an Amazon Research Award. Its views and conclusions are those 
of the authors only.



\bibliography{lib.bib}
\bibliographystyle{icml2022}

\newpage
\appendix
\onecolumn
\setcounter{theorem}{0}
\setcounter{corollary}{0}
\setcounter{lemma}{0}
\setcounter{proposition}{0}

\section{Missing proofs}
\label{app:proofs}

\begin{proposition}
    \label{app:thm:delta_vs_entropy}
    Given a core feature set $R \subseteq S$ with failure probability $\delta <0.5$, then there exists a function $\epsilon:\mathbb{R}\to\mathbb{R}$ that is  monotonic decreasing function with $\epsilon(0)=0$ such that: 
  \[
      H\big[ f_\theta( X_U, X_R = x_R ) \big] \leq \epsilon(\delta),
  \]
where $H[Z] \text{=} -\sum_{z \in [L]}\Pr(Z=z) \log \Pr(Z=z)$ 
is the entropy of the random variable $Z$. 
\end{proposition}

\begin{proof}
In this proof, we demonstrate the binary classification case. The extension to a multi-class scenario can be achieved through a similar process.

By the definition of the core feature set, there exists a representative label, denoted as $\tilde{y} \in \{0,1\}$ such that the probability of $P(f_{\theta}(X_U, X_R =x_R) = \tilde{y})$ is greater than or equal to $1 -\delta $. 
Without loss of generality, we assume that the representative label is $\tilde{y} =1$. 
Therefore, if we denote $Z$ as the probability of $Pr( f_{\theta}(X_U, X_R =x_R) =1) $, then the probability of $Pr(f_{\theta}(X_U, X_R =x_R) =0) = 1-Z$. 
Additionally, we have $Z \geq 1-\delta > 0.5$ due to the assumption that $\delta <0.5$. The entropy of the model's prediction can be represented as: 
$H\big[ f_\theta( X_U, X_R = x_R ) \big]  = -Z \log Z - (1-Z) \log (1-Z) $.

Let $\epsilon(Z) = -Z \log Z - (1-Z) \log (1-Z) $. 
The derivative of $\epsilon(Z)$ is given by $ \frac{d\epsilon(Z)}{dZ} = \log \frac{1-Z}{Z} <0 $, as $Z > 0.5$. As a result, $\epsilon(Z)$ is a monotonically decreasing function.

When $\delta =0$, we have $Z=1$, and by the property of the entropy $ H\big[ f_\theta( X_U, X_R = x_R ) \big]  = 0$. 
\end{proof}

\begin{proposition}
    \label{thm:cond_entropy}
    Given two subsets $R$ and $R'$ of sensitive features $S$, with 
    $R \subseteq R'$, 
    \[
      H\big( f_\theta(X_U, X_R=x_R) \big) \geq 
      H\big( f_\theta(X_{U'}, X_{R'}=x_{R'}) \big),
    \]
    where $U=S\setminus R$ and $U'=S\setminus R'$.
\end{proposition}

\begin{proof}
This is due to the property that conditioning reduces the  uncertainty, or the well-known \emph{information never hurts}  theorem in information theory \cite{krause2005note}.
\end{proof}

\begin{proposition}
\label{app:prop:2} The conditional distribution of any subset of unrevealed features 
$U' \in U$, given the the values of released features $X_R =x_R$ is given by:
\begin{align*}
\Pr(X_{U'} | X_R = x_R)  &= 
    \mathcal{N}\bigg(\mu_{U'}  + \Sigma_{U', R} \Sigma^{-1}_{R, R} 
        (x_R - \mu_R),\;
        \Sigma_{U',U'} - \Sigma_{U',R}\Sigma^{-1}_{R,R} \Sigma_{R,U'} \bigg),
\end{align*}
where $\Sigma$ is the covariance matrix 
 \end{proposition}

 \begin{proof}
 This is a well-known property of the Gaussian distribution and we refer the reader to Chapter 2.3.2 of the textbook \cite{bishop2006pattern} for further details. 
 \end{proof}


\begin{proposition}
\label{app:prop:3}
The model predictions before thresholding, $\tilde{f}_{\theta}(X_U, X_R = x_R) =\theta_U X_U + \theta_R x_R $ 
is a random variable with a Gaussian distribution $\mathcal{N}\big(m_f, \sigma_f \big)$, 
where
\begin{align}
    m_f &= \theta_R x_R + \theta^\top_U \big(  \mu_U + \Sigma_{U, R} \Sigma^{-1}_{R,R} ( x_{R} - \mu_R) \big) \label{eq:app_prop4_a}\\ 
    \sigma^2_f &= \theta^\top_{U} \big( \Sigma_{U,U} - \Sigma_{U, R} \Sigma^{-1}_{R,R} \Sigma_{R,U} \big) \theta_{U}, \label{eq:app_prop4_b}
\end{align}
where $\theta_U$ is the sub-vector of parameters $\theta$ corresponding
to the unrevealed features $U$. 
 \end{proposition} 

\begin{proof}
The proof of this statement is straightforward due to the property that a linear combination of Gaussian variables $X_U$ is also Gaussian. Additionally, the posterior distribution of $X_U$ is already provided in Proposition  \ref{prop:2}.
\end{proof}

\begin{proposition}
\label{app:prop:4}
Let the model predictions prior thresholding $\tilde{f}_\theta(X_U, X_R =x_R)$, 
be a random variable following a Gaussian distribution $\mathcal{N} (m_f, 
\sigma^2_f)$. Then, the model prediction following thresholding $f_{\theta}( X_U, X_R =x_R)$ 
is a random variable following a Bernoulli distribution $Bern(p)$ 
with $p = \Phi( \frac{m_f}{\sigma_f})$, where $\Phi(\cdot)$ refers to 
the  CDF of the standard normal distribution, and $m_f$ and $\sigma_f$,
are given in Equations \eqref{eq:prop4_a} and \eqref{eq:prop4_b}, respectively.
\end{proposition}

\begin{proof}
In the case of a binary classifier, we have $f_{\theta}(x) = \bm{1}\{ \tilde{f}_{\theta}(x) \geq 0 \}$. 
If $\tilde{f}$ follows a normal distribution, denoted as $\tilde{f} \sim \mathcal{N}(m_f, \sigma^2_f)$, then by the properties of the normal distribution, $ f{\theta} $ follows a Bernoulli distribution, denoted as $f_{\theta} \sim Bern(p)$, with parameter $p =\Phi( \frac{m_f}{\sigma_f})$, where $\Phi(\cdot)$ is the cumulative density function of the standard normal distribution.
\end{proof}

\begin{proposition}
\label{app:thm:core_linear_verf}
Assume $f_\theta$ is a linear classifier. Then, determining if a 
subset $U$ of sensitive features $S$ is a \emph{pure} core feature set 
can be performed in $O(|P| + |S|)$ time.
\end{proposition} 

\begin{proof}
As discussed in the main text, to test if a subset $U$ is a core feature set or not, we need to check  if the following two terms have the same sign (either negative or non-negative):
\begin{align}
\begin{split}
\max_{X_U} \,  & \theta^\top_U X_U + \theta^\top_R x_R = \|\theta\|_1 + \theta^\top_R x_R \\
\min_{X_U} \, & \theta^\top_U X_U + \theta^\top_R x_R = - \|\theta\|_1 + \theta^\top_R x_R.
\label{eq:min_max_linear}
\end{split}
\end{align}
These can be solved in time $O( |P| + |S|)$ due to the property of the linear equality above.
\end{proof}

\begin{theorem}
\label{thm:taylor_approx}
The distribution of the  random variable $\tilde{f}_\theta = \tilde{f}_{\theta}(X_U, X_R =x_R)$ where $X_U \sim \mathcal{N}\big(\mu^{\mbox{pos}}_{U}, \Sigma^{\mbox{pos}}_{U} \big)$ can be  approximated by a Normal distribution as
\begin{align}
    \tilde{f}_\theta \sim \mathcal{N} \big( \tilde{f}_{\theta}( X_U = \mu^{pos}_U, X_R = x_R) , g^\top_{U}\Sigma^{pos}_{U}g_U\big) 
    \label{eq:taylor_arppox}
\end{align}
where $g_{U} = \nabla_{X_U}\tilde{f}_{\theta}(  X_U = \mu^{pos}_{U}, X_R = x_R)$ is the gradient of model prediction at $X_U = \mu^{pos}_{U}$. 
\end{theorem}

\begin{proof}
The proof relies on the first Taylor approximation of classifier $\tilde{f}$ around its mean: 
\begin{align}
 \tilde{f}_{\theta}(X_U, X_R = x_R, ) &\approx \tilde{f}_{\theta}( X_U = \mu^{pos}_{U}, X_R = x_R)  + (X_U - \mu^{pos}_U)^T \nabla_{X_U}\tilde{f}_{\theta}(X_U = \mu^{pos}_U ,  X_R = x_R)
 \label{app:eq:taylor_approx}
\end{align}
Since $X_U \sim \mathcal{N}\big(\mu^{\mbox{pos}}_{U}, \Sigma^{\mbox{pos}}_{U} \big) $ hence $X_{U} -\mu^{\mbox{pos}}_{U} \sim  \mathcal{N}\big(\boldsymbol{0}, \Sigma^{\mbox{pos}}_{U} \big) $. By the properties of normal distribution, the right-hand side of Equation (\ref{app:eq:taylor_approx}) is a linear combination of Gaussian variables, and it is also Gaussian.
\end{proof}

\section{Algorithms Pseudocode}
\label{app:pseudocodes}
The pseudocode for PFR for non-linear classifiers is presented in Algorithm \ref{alg:alg2}. There are two main differences between this algorithm and the case of linear classifiers. Firstly, the procedure of pure core feature testing on line 5 takes exponential time with respect to $|U|$ instead of linear time as in the case of linear classifiers. Additionally, we use Theorem \ref{thm:taylor_approx} to estimate the distribution of the soft prediction as seen on line 11, as the exact distribution cannot be computed analytically as in the case of linear classifiers.

\begin{algorithm}[t]
  \caption{PFR for non-linear classifiers}
  \label{alg:alg2}
  \setcounter{AlgoLine}{0}
  \SetKwInOut{Input}{input}
  \SetKwInOut{Output}{output}
  \SetKwRepeat{Do}{do}{while}
  \Input{A test sample $x$; Training data $D$} 
  \Output{A core feature set $R$  and its representative label $\tilde y$}
  $\mu \gets \frac{1}{|D|} \sum_{(x,y) \in D} x$\\
  $\Sigma  \gets \frac{1}{|D|} \sum_{(x,y) \in D} (x -\mu) (x - \mu)^\top$\\ 
  Initialize $R = \emptyset$\\
  \While{True} {
    \eIf{$R$ is a core feature set with repr.~label $\tilde y$}{
      \Return{$(R, \tilde y)$}
    }{
    \ForEach{$j \in U$} {
       Compute $\Pr(X_j | X_R = x_R)$  (using Prop.~\ref{prop:2})\\
      $\bm Z \gets \text{sample}(\Pr(X_j |X_R = x_R))$ T times\\
       Compute $\Pr\left(f_{\theta}(X_j = z, X_{U\setminus \{j\}}X_R = x_R)\right)$  \, ( using Theorem \ref{thm:taylor_approx})\\
       Compute $F(X_j)$ (using Eq. \eqref{eq:monte_carlo})
       }
    }
     $j^* \gets \argmax_j F(X_j)$\\
     $R \gets R \cup \{j^*\}$\\
     $U \gets U \setminus \{j^*\}$\\
    }
\end{algorithm}

\section{Extension from binary to multiclass classification}
\label{app:multiclass}
In the main text, we provide the implementation of PFR  for binary classification problem.  In this section, we extend the method to the multiclass classification problem.
 
\subsection{Estimating $P(f_{\theta}(X_U, X_R =x_R))$}

In order to achieve our goals of determining if a subset is a core feature set for a given $\delta >0$, and computing the entropy in the scoring function, we need to estimate the distribution of $f_{\theta}(X_U, X_R =x_R)$. In this section, we first discuss the method of computing the distribution of $\tilde{f}{\theta}(X_U, X_R =x_R)$ for both linear and non-linear models. Once this is done, we then address the challenge of estimating the hard label distribution $P(f{\theta}(X_U, X_R =x_R))$.

It is important to note that, under the assumption that the input features $X$ are normally distributed with mean $\mu$ and covariance matrix $\Sigma$, the linear classifier $\tilde{f}_{\theta} = \theta^\top {x}$ will also have a multivariate normal distribution. Specifically, if $X_U \sim \mathcal{N}(\mu^{pos}_U, \Sigma^{pos}_U)$, then $\tilde{f}{\theta}(X_U, X_R =x_R) \sim \mathcal{N}(\theta^\top_R x_R + \theta^T_{U} \mu^{pos}_U, \theta^\top_U \Sigma \theta_U)$.

For non-linear classifiers, the output $f_{\theta}(X_U, X_R =x_R) $ is not a Gaussian distribution due to the non-linear transformation. To approximate it, we use Theorem \ref{thm:taylor_approx} which states that the non-linear function $\tilde{f}_{\theta}(X_U, X_R =x_R)$ can be approximated as a multivariate Gaussian distribution.

\paragraph{Challenges when estimating $P(f_{\theta}(X_U, X_R =x_R))$} 
For multi-class classification problems, the hard label $f_{\theta}(X_U, X_R =x_R)$ is obtained by selecting the class with the highest score, which is given by $\argmax_{i \in [L]} \tilde{f}^i_{\theta}(X_U, X_R =x_R)$. However, due to the non-analytical nature of the $\argmax$ function, even when $\tilde{f}{\theta}(X_U, X_R =x_R)$ follows a Gaussian distribution, the distribution of $f{\theta}(X_U, X_R =x_R)$ cannot be computed analytically. To estimate this distribution, we resort to Monte Carlo sampling. Specifically, we draw a number of samples from $P(f_{\theta}(X_U, X_R =x_R))$, and approximate the probability of each class as the proportion of samples that fall in that class.

We provide experiments of PFR for multi-class classification cases in Section \ref{sec:app_multi_class_exp}.

\section{Experiments details}
\label{sec:app_exp}

\paragraph{Datasets information}
To show the advantages of the suggested PFR technique for safeguarding feature-level privacy, we employ benchmark datasets in our experiments. These datasets include both binary and multi-class classification datasets. The following are examples of binary datasets that we use to evaluate the method:
\begin{enumerate}
    \item Bank dataset \cite{UCIdatasets}.  The objective of this task is to predict whether a customer will subscribe to a term deposit using data from various features, including but not limited to call duration and age. There are a total of 16 features available for this analysis.
    
    \item Adult income dataset \cite{UCIdatasets}. The goal of this task is to predict whether an individual earns more than \$50,000 annually. After preprocessing the data, there are a total of 40 features available for analysis, including but not limited to occupation, gender, race, and age.
    
    \item Credit card default dataset \cite{UCIdatasets}.  
    The objective of this task is to predict whether a customer will default on a loan. The data used for this analysis includes 22 different features, such as the customer's age, marital status, and payment history.
    
     \item Car insurance  dataset \cite{car_insurance}. 
     The task at hand is to predict whether a customer has filed a claim with their car insurance company. The dataset for this analysis is provided by the insurance company and includes 16 features related to the customer, such as their gender, driving experience, age, and credit score.
\end{enumerate}

Furthermore, we also evaluate our method on two additional multi-class classification datasets:
\begin{enumerate}
\item Customer segmentation dataset  \cite{customer}. 
The task at hand is to classify a customer into one of four distinct categories: A, B, C, and D. The dataset used for this task contains 9 different features, including profession, gender, and working experience, among others.

\item Children fetal health dataset \cite{fetal}. The task at hand is to classify the health of a fetus into one of three categories: normal, suspect, or pathological, using data from CTG (cardiotocography) recordings. The data includes approximately 21 different features, such as heart rate and the number of uterine contractions.
\end{enumerate}

\paragraph{Settings:} 
For each dataset, 70\% of the data will be used for training the classifiers, while the remaining 30\% will be used for testing. The number of sensitive features, denoted as $|S|$, will be chosen randomly from the set of all features, with $|S|$ ranging from 2 to 7. The remaining features will be considered as public. 100 repetition experiments will be performed for each choice of $|S|$, under different random seeds, and the results will be averaged. All methods that require Monte Carlo sampling will use 1000 random samples. The performance of different methods will be evaluated based on accuracy and information leakage. Two different classifiers will be considered.

\begin{enumerate}
    \item Linear classifiers: We use Logistic Regression as the base classifier.
    \item Nonlinear classifiers: The nonlinear classifiers used in this study consist of a neural network with two hidden layers, using the ReLU activation function. The number of nodes in each hidden layer is set to 10. The network is trained using stochastic gradient descent (SGD) with a batch size of 32 and a learning rate of 0.001 for 300 epochs. A value of $\Delta$ = 0.2 is used when testing the pure core feature set for nonlinear classifiers. 
\end{enumerate}

\paragraph{Baseline models.} 
We compare our proposed algorithms with the following baseline models: 
\begin{enumerate}
\item \textbf{Baseline}: This refers to the usage of original classifier which asks users to reveal \textbf{all} sensitive features. 
\item \textbf{Opt}: This method involves evaluating all possible subsets of sensitive features in order to identify the minimum \emph{pure} core feature set. For each subset, the verification algorithm is used to determine whether it is a pure core feature set. The minimum pure core feature set that is found is then selected. It should be noted that as all possible subsets are evaluated, all sensitive feature values must be revealed. Therefore, this approach is not practical in real-world scenarios. However, it does provide a lower bound on information leakage for PFR (when $\delta=0$).
\end{enumerate}

\paragraph{Metrics.} We compare all different algorithms in terms of accuracy and information leakage:
\begin{enumerate}
\item Accuracy. For algorithms that are based on the core feature set, such as our PFR and "Opt," the representative label is used as the model's prediction. The accuracy is then determined by comparing this label to the ground truth.

\item Information leakage. We compute the average number of sensitive features that need to be revealed over the test set. A smaller number is considered better. 
\end{enumerate}

\begin{figure*}[tb]
\captionsetup{justification=centering}
\centering
\begin{subfigure}[b]{0.43\textwidth}
\includegraphics[width = 1.0\linewidth]{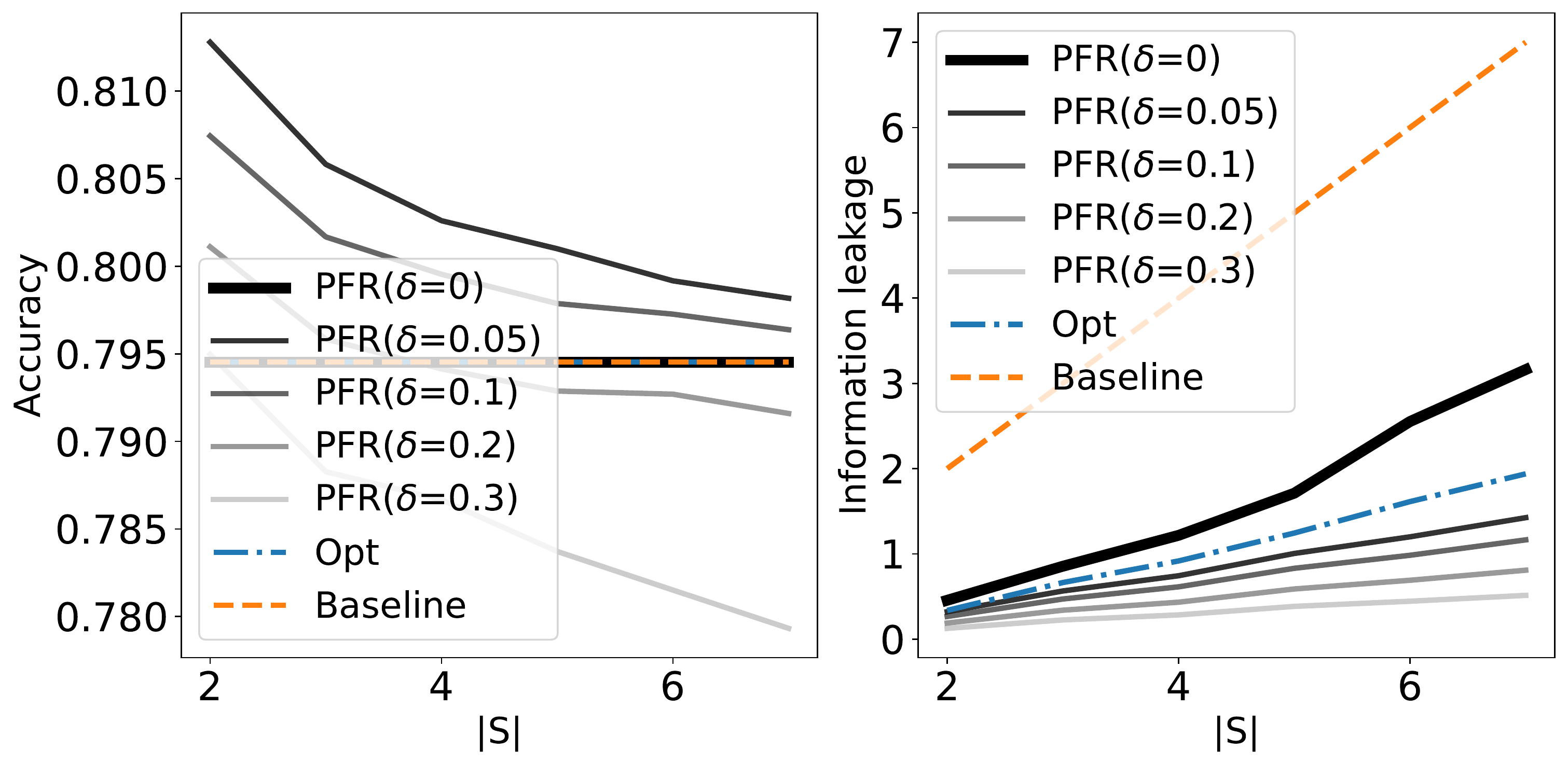}
\caption{Bank dataset}
\end{subfigure}
\begin{subfigure}[b]{0.43\textwidth}
\includegraphics[width = 1.0\linewidth]{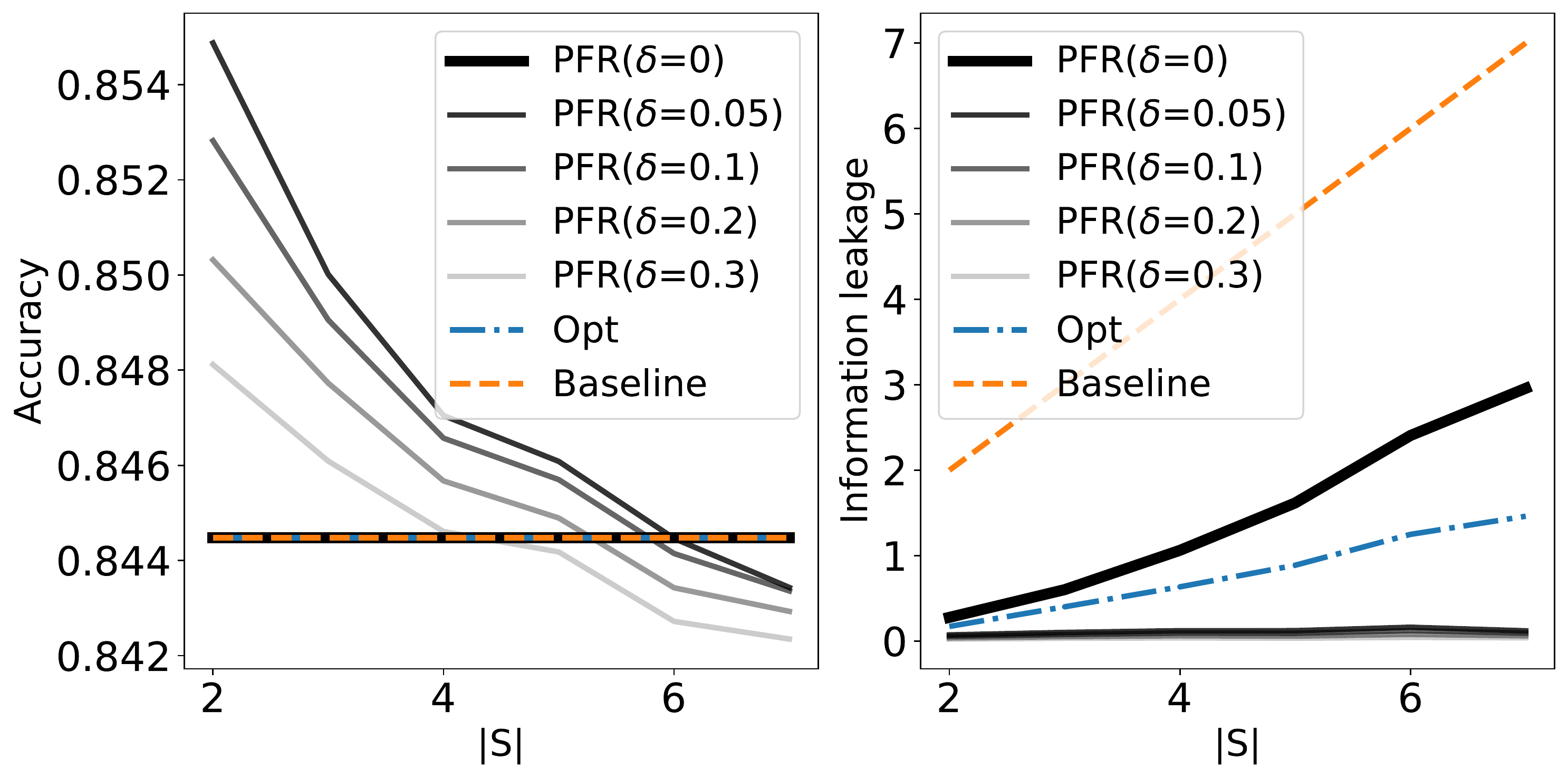}
\caption{Adult income dataset}
\end{subfigure}
\caption{Accuracy and information leakage for different choices of number of private features $m$ under Logistic Regression classifiers}
\label{fig:compare_linear_app}
\end{figure*}

\begin{figure*}[tb]
\captionsetup{justification=centering}
\centering
\begin{subfigure}[b]{0.43\textwidth}
\includegraphics[width = 1.0\linewidth]{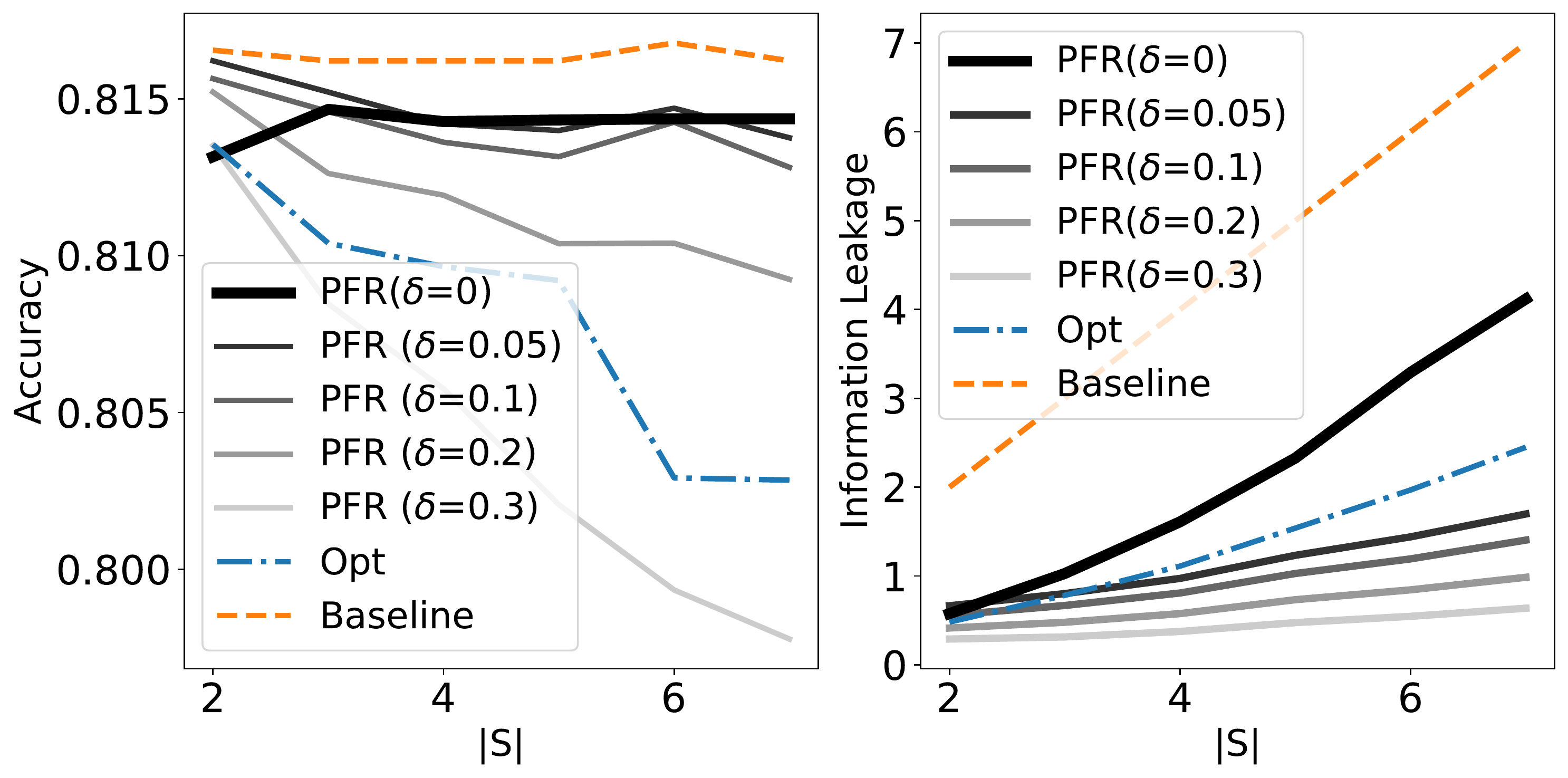}
\caption{Bank dataset}
\end{subfigure}
\begin{subfigure}[b]{0.43\textwidth}
\includegraphics[width = 1.0\linewidth]{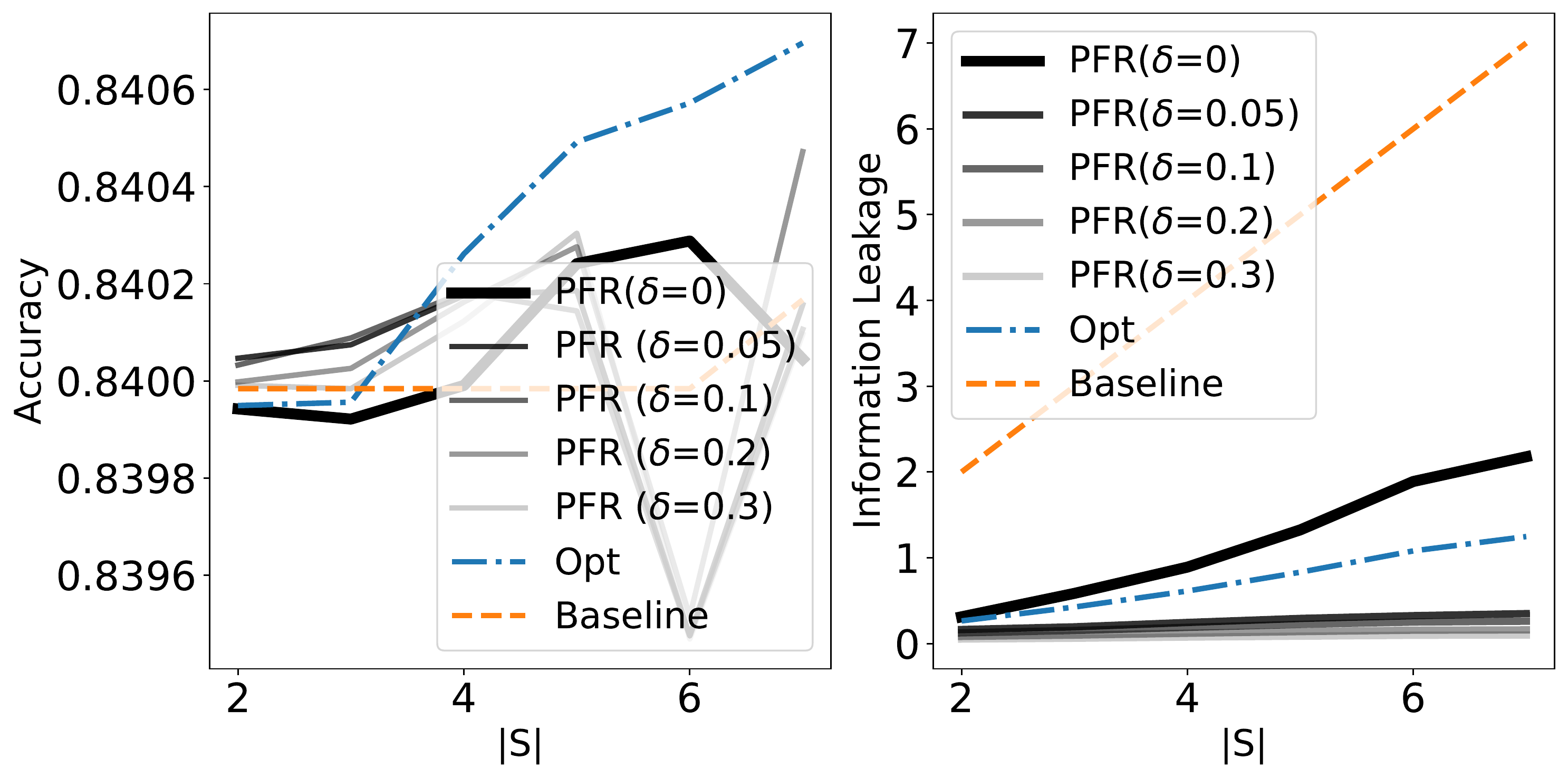}
\caption{Adult income dataset}
\end{subfigure}
\caption{Accuracy and information leakage for different choices of number of sensitive features $|S|$ under non-linear classifiers}
\label{fig:compare_nonlinear_app}
\end{figure*}

\subsection{Additional experiments on linear binary classifiers}
Additional experiments were conducted to compare the performance of PFR to that of the baseline methods using linear classifiers on the Bank and Adult income datasets, as shown in Figure \ref{fig:compare_linear_app}. As in the main text, a consistent trend in terms of performance is observed. As the number of sensitive attributes, $|S|$, increases, the information leakage introduced by PFR with various values of $\delta$ increases at a slower rate. With different choices of $|S|$, PFR (with $\delta = 0$) requires the revelation of at most 50\% of sensitive information. To significantly reduce the information leakage of PFR, the value of $\delta$ can be relaxed. By choosing an appropriate value for the failure probability, such as $\delta = 0.1$, only minimal accuracy is sacrificed (at most 0.002\%), while the information leakage can be reduced to as low as 5\% of the total number of sensitive attributes.

\subsection{Additional experiments on non-linear binary classifiers}

Additional experiments were conducted to compare the performance of PFR to that of the baseline methods using non-linear classifiers on the Bank and Adult income datasets, as shown in Figure \ref{fig:compare_nonlinear_app}. As seen, while the Baseline method requires the revelation of all sensitive attributes, PFR with different values of $\delta$ only requires the revelation of a much smaller number of sensitive attributes. The accuracy difference between the Baseline method and PFR is also minimal (at most 2\%). These results demonstrate the effectiveness of PFR in protecting privacy while maintaining a good prediction performance for test data.

\subsection{Evaluation of PFR on multi-class classifiers}
\label{sec:app_multi_class_exp}
\paragraph{Linear classifiers}
We also provide a comparison of accuracy and information leakage between our proposed FPR and the baseline models for linear classifiers. These metrics are reported for the Customer and Children Fetal Health datasets in Figure \ref{fig:multi_class_linear}. The figure clearly shows the benefits of FPR in reducing information leakage while maintaining a comparable accuracy to the baseline models.

\begin{figure*}[tb]
\captionsetup{justification=centering}
\centering
\begin{subfigure}[b]{0.43\textwidth}
\includegraphics[width = 1.0\linewidth]{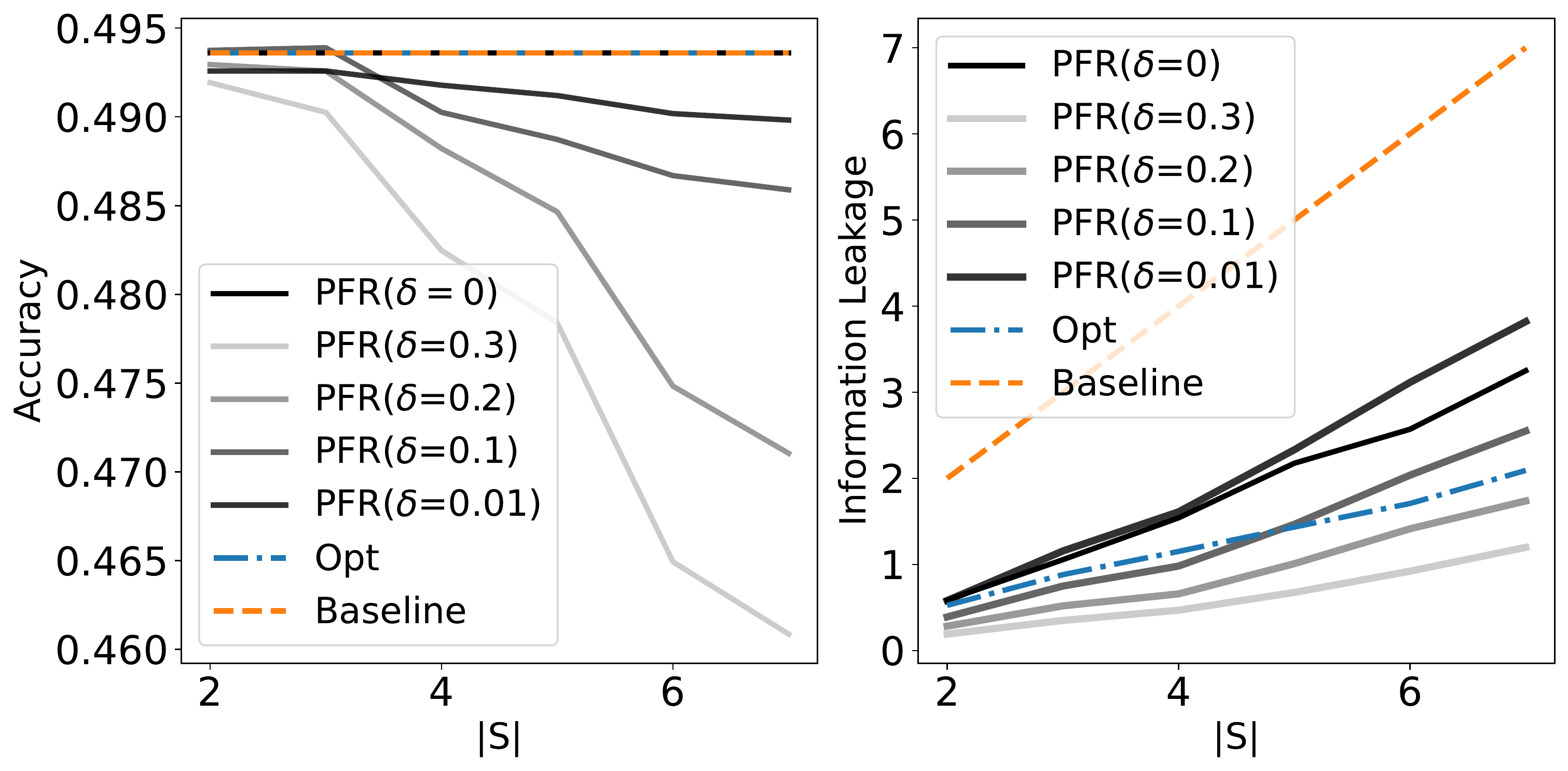}
\caption{Customer dataset}
\end{subfigure}
\begin{subfigure}[b]{0.43\textwidth}
\includegraphics[width = 1.0\linewidth]{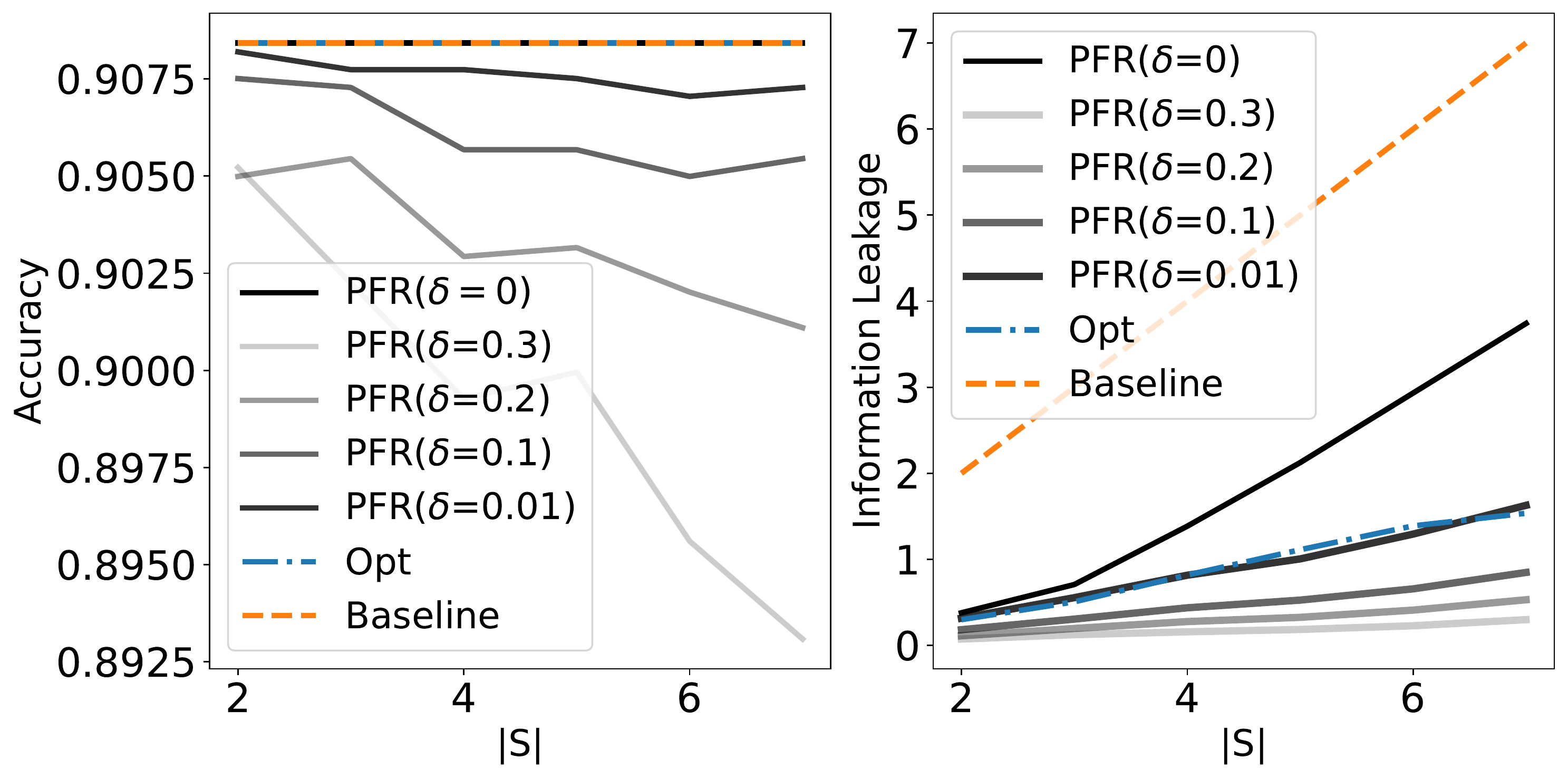}
\caption{Children fetal health dataset}
\end{subfigure}
\caption{Accuracy and information leakage for different choices of number of sensitive features $|S|$ under multinomial Logistic Regression}
\label{fig:multi_class_linear}
\end{figure*}

\paragraph{Nonlinear classifiers}
Similarly, we present a comparison of our proposed algorithms with the baseline methods when using non-linear classifiers. These metrics are reported for the Customer and Children Fetal Health datasets in Figure \ref{fig:multi_class_nonlinear}. The results show that using PFR with a value of $\delta=0$ results in a minimal decrease in accuracy, but significantly reduces the amount of information leakage compared to the Baseline method.

\begin{figure*}[tb]
\captionsetup{justification=centering}
\centering
\begin{subfigure}[b]{0.43\textwidth}
\includegraphics[width = 1.0\linewidth]{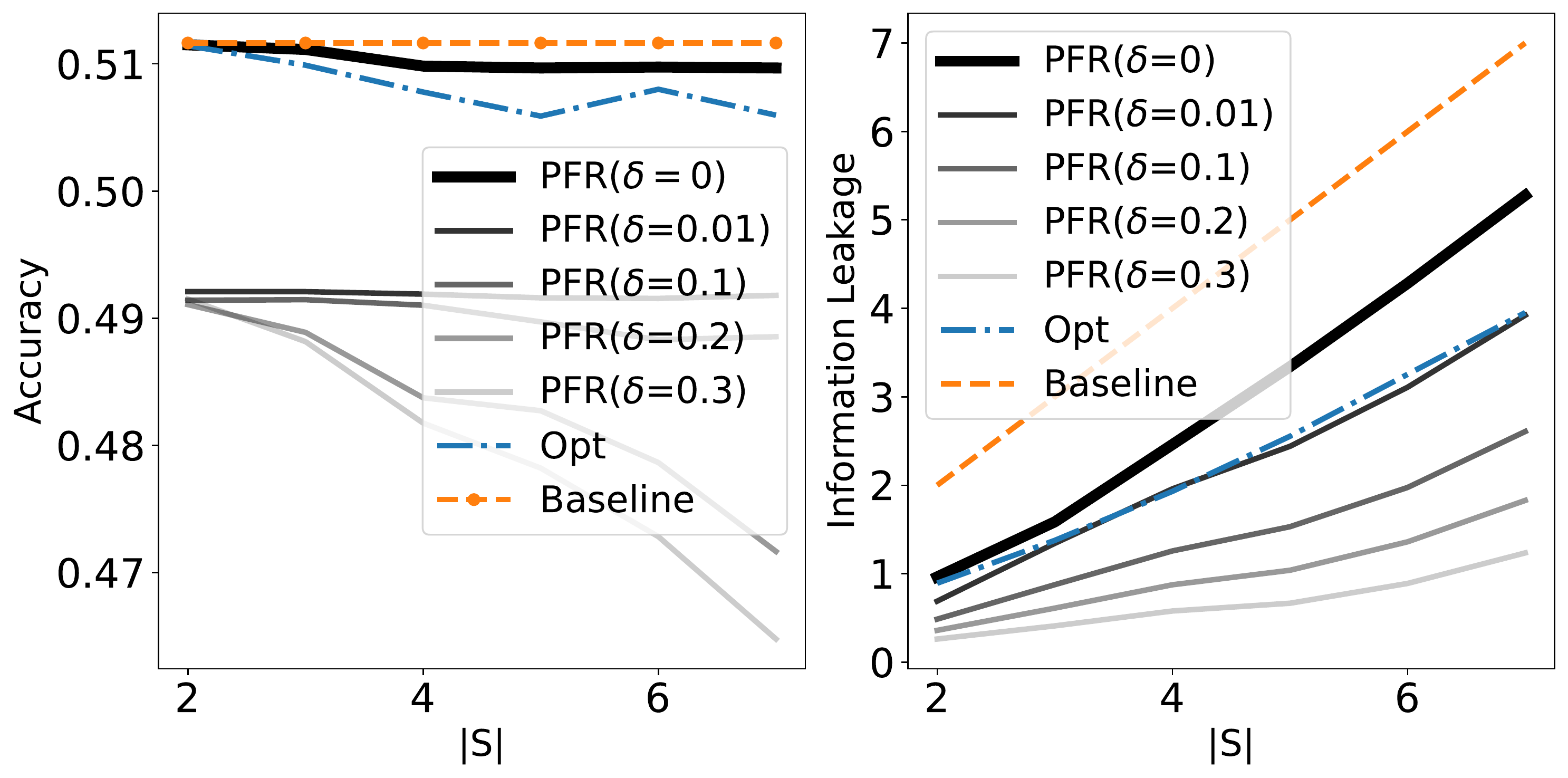}
\caption{Customer segmentation dataset}
\end{subfigure}
\begin{subfigure}[b]{0.43\textwidth}
\includegraphics[width = 1.0\linewidth]{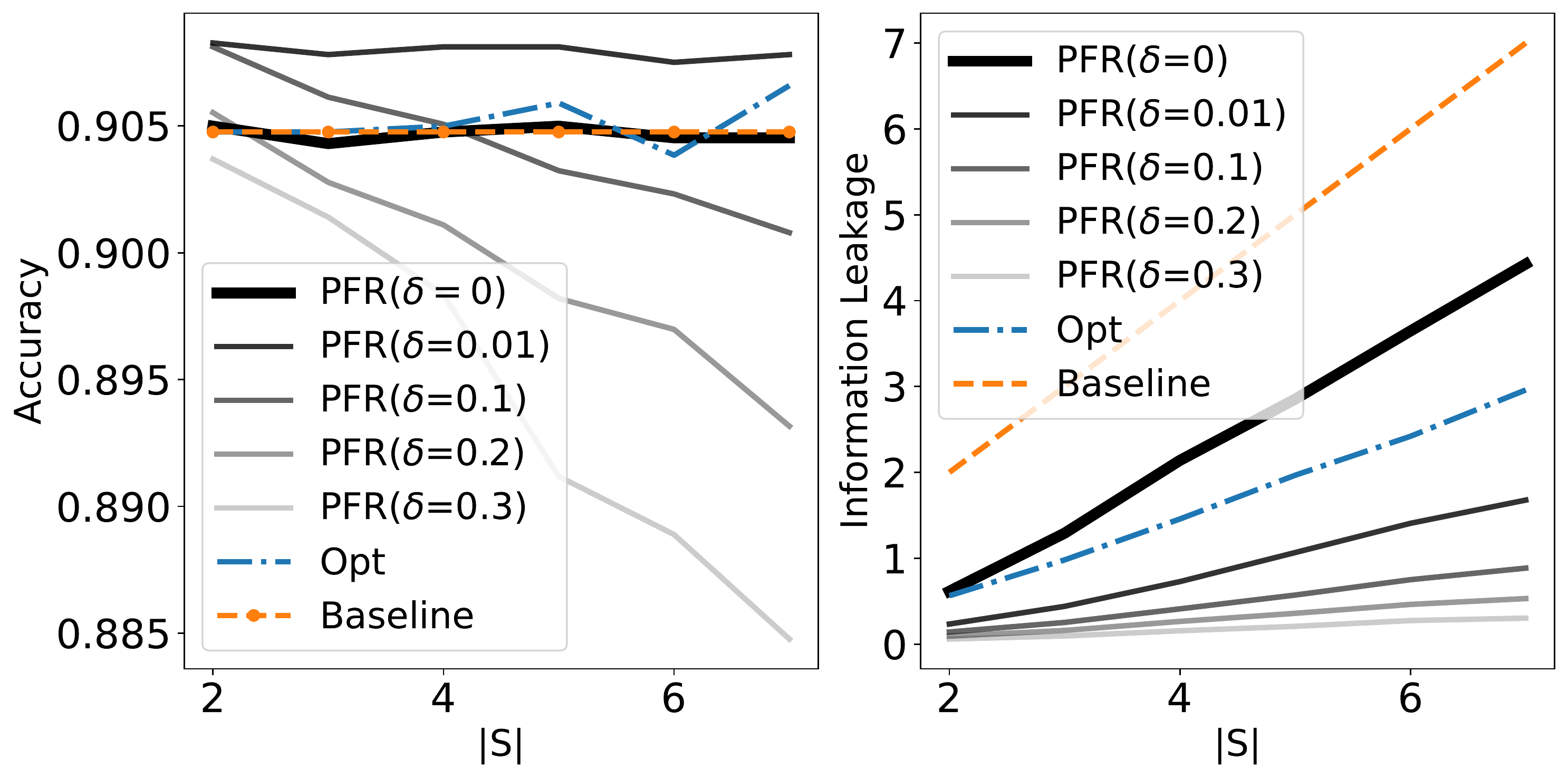}
\caption{Children fetal health dataset}
\end{subfigure}
\caption{Accuracy and information leakage for different choices of number of sensitive features $|S|$ under non-linear classifiers}
\label{fig:multi_class_nonlinear}
\end{figure*}

\end{document}